\pdfoutput=1 

\documentclass[11pt]{article}
\usepackage{cpdef} 

\usepackage[preprint]{acl}

\usepackage{times}
\usepackage{latexsym}

\usepackage[T1]{fontenc}

\usepackage[utf8]{inputenc}

\usepackage{microtype}
\usepackage{tcolorbox}

\usepackage{inconsolata}
\usepackage{graphicx}

\usepackage{diagbox}
\usepackage{graphicx}
\usepackage{hyperref}
\usepackage{url}
\usepackage{booktabs}
\usepackage{multicol}
\usepackage{ulem}
\usepackage{tikz}
\usepackage{fontawesome5}
\definecolor{MapleRed}{RGB}{191,5,5} %

\usepackage{listings}
\usepackage{xcolor}
\definecolor{mygray}{gray}{0.97}
\definecolor{promptblue}{rgb}{0,0.4,0.8}
\definecolor{promptgreen}{rgb}{0.3,0.6,0.3}
\definecolor{promptpurple}{rgb}{0.6,0,0.6}

\lstdefinestyle{prompt}{
    backgroundcolor=\color{mygray},
    basicstyle=\normalfont,
    breaklines=true,
    columns=fullflexible,
    frame=single,
    keepspaces=true,
    numbers=left,
    numberstyle=\tiny\color{gray},
    keywordstyle=\color{blue},
    commentstyle=\color{green!60!black},
    stringstyle=\color{purple},
    escapeinside={(*@}{@*)},
    breakatwhitespace=false,
    moredelim=[is][\color{red}\bfseries]{<<!}{!>>},       %
    moredelim=[is][\color{promptblue}]{<<*}{*>>},         %
    moredelim=[is][\colorbox{yellow!30}]{<bg>}{</bg>},    %
    moredelim=[is][\color{promptgreen}]{<<\#\#}{\#\#>>},  %
    moredelim=[is][\color{promptpurple}\itshape]{<<@}{@>>} %
}

\title{\ours \textcolor{MapleRed}{\faIcon{canadian-maple-leaf}}: Multi-Agent Post-Co-Training for Collaborative \\Large Language Models with Reinforcement Learning}

\author{Chanwoo Park$^\heartsuit$\thanks{This work was initiated during an internship at Amazon AGI.}\hspace{0.9cm}Seungju Han$^\spadesuit$\hspace{0.9cm}Xingzhi Guo$^\diamondsuit$\\
\textbf{\hspace{0cm}Asuman Ozdaglar$^\heartsuit$\hspace{0.66cm}Kaiqing Zhang$^\clubsuit$\hspace{0.48cm}
Joo-Kyung Kim$^\diamondsuit$} \\[0.1cm]
$^\heartsuit$MIT \hspace{0.8cm}$^\spadesuit$Stanford\hspace{0.8cm}$^\diamondsuit$Amazon \hspace{0.8cm}$^\clubsuit$UMD\\
\texttt{{\{cpark97,asu\}@mit.edu}, seungju@stanford.edu}, \\
\texttt{\{guxingzh, jookyk\}@amazon.com, kaiqing@umd.edu}
}

\newboolean{finalver}
\setboolean{finalver}{true} %

\ifthenelse{\boolean{finalver}}{
      \newcommand{\cp}[1]{\textcolor{purple}{\bf\small [#1 --CP]}}
    \newcommand{\kz}[1]{\textcolor{green}{\bf\small [#1 --KZ]}}
  \newcommand{\jk}[1]{\textcolor{brown}{\bf\small [#1 --JK]}}
  \newcommand{\jacky}[1]{\textcolor{red}{\bf\small [#1 --Jacky]}}
  \newcommand{\sj}[1]{\textcolor{red}{\bf\small [#1 --SJ]}}

\newcommand{\safevspace}[1]{\vspace{0mm}}
}{
  \newcommand{\cp}[1]{\textcolor{purple}{\bf\small}}
  \newcommand{\jk}[1]{\textcolor{brown}{\bf\small}}
  \newcommand{\jacky}[1]{\textcolor{red}{\bf\small}}
  \newcommand{\sj}[1]{\textcolor{red}{\bf\small}}
      \newcommand{\kz}[1]{}

\newcommand{\safevspace}[1]{\vspace{0mm}}
}

\newcommand{\oursspace}{\textbf{\texttt{MAPoRL}} }
\newcommand{\ours}{\textbf{\texttt{MAPoRL}}}
\newcommand{\ourstwo}{\textbf{\texttt{MAPoRL}}}
\begin{document}
\maketitle
\begin{abstract}
Leveraging multiple large language models (LLMs) to build collaborative  {\it multi-agentic} workflows has demonstrated significant potential. However, most previous studies focus on {\it prompting}  the out-of-the-box LLMs, relying on their innate capability for collaboration, which may not improve LLMs' performance as shown recently.
In this paper, we introduce a new {\it post-training} paradigm {\bf \ours} ({\bf M}ulti-{\bf A}gent {\bf P}ost-c{\bf o}-training for collaborative LLMs with {\bf R}einforcement {\bf L}earning), to explicitly elicit the collaborative behaviors and further unleash the power of multi-agentic LLM frameworks. 
In {\bf \ours}, multiple LLMs first generate their own responses independently and engage in a multi-turn discussion to collaboratively improve the final answer. 
In the end,  a  {\bf \ours} verifier evaluates both the answer and the discussion, by assigning a score that verifies the correctness of the  answer,  while  adding incentives to encourage corrective and persuasive discussions. The score serves as the co-training reward, and is then maximized through multi-agent RL.
Unlike existing LLM post-training paradigms, {\bf \ours} advocates the {\it co-training} of multiple LLMs together using {\it RL} for better generalization. Accompanied by analytical insights, our experiments demonstrate that training individual LLMs alone is insufficient to induce effective collaboration. In contrast, multi-agent co-training can boost  the collaboration performance across benchmarks, with generalization to unseen domains. The code is available at \url{https://github.com/chanwoo-park-official/MAPoRL}.

\end{abstract}
\safevspace{-4mm}
\section{Introduction}

\safevspace{-2.2mm}
Recent advances in large language models (LLMs) have highlighted their potential for collaboration, particularly within the {\it multi-agentic} framework  \citep{du2023improving, li2023camel, kim2024adaptive}. The shift from single-agent to multi-agent systems introduces new dimensions and challenges in enabling effective collaboration among LLM agents. 
Recent approaches to  multi-LLM  collaboration mostly rely on 
{{\it prompting} {\it pre-trained} models. However,} such approaches struggle with achieving genuine collaboration among the agents. For example, multi-agent debate does not consistently lead to improved performance with additional turns \citep{huang2023large}.

\safevspace{-1.5mm}

This limitation may be somewhat expected -- while LLMs are able to {\it simulate} collaboration procedures, they were  {\it not} explicitly {\it trained} to achieve effective cooperation.
{In theory, it is not hard to imagine that single-agent training is insufficient for collaboration -- an {\it untrained} and {\it non-strategic} opponent can fail to act in a way that promotes  collaboration. 
Instead,} 
achieving collaborative behaviors requires interactive training environments where each agent actively engages with others, and dynamically optimizes the strategy \citep{gagne1974instruction,macy1991learning, hertz2013learning}. Moreover, conventional approaches such as supervised fine-tuning (SFT), as we will show, are inadequate for this purpose, either:  
merely mimicking multi-agent interactions from training data may not lead to effective collaboration. 

\safevspace{-1.5mm}

To develop more effective collaborative agents, we propose {\bf M}ulti-{\bf A}gent {\bf P}ost-c{\bf o}-training for collaborative LLMs with {\bf R}einforcement {\bf L}earning (\textbf{\ours}), a {\it co-training} paradigm for multiple LLMs using multi-agent reinforcement learning (MARL). In \ourstwo, within the  pre-defined frameworks for multi-agent collaboration ({e.g.}, the debate framework \citep{du2023improving}), each agent receives rewards for their responses during collaboration, based on the quality of their answers and interactions. The objective for each agent in  \ourstwo~ is to maximize their own value function, defined as the expected cumulative sum of rewards over the course of the collaboration.

\safevspace{-1.5mm}
To further encourage cooperation in \ourstwo, we incorporate incentives for successful interactions and penalties for collaboration failures, steering the LLMs toward more effective and aligned behaviors. Through a simplified game-theoretic example, we validate the following insights: 1) single-agent training alone is insufficient to produce genuinely cooperative agents, and 2) co-trained agents can reach an equilibrium that exhibits cooperative behavior.

\safevspace{-1.5mm}
To assess the effectiveness of \ourstwo, we conduct  experiments across diverse tasks and evaluation strategies. Specifically, we train multi-agent LLMs for tasks such as mathematical reasoning (GSM8k \citep{cobbe2021training}) and natural language inference (ANLI \citep{nie2019adversarial}), comparing their performance against baseline approaches. Additionally, we evaluate the robustness of our method by testing agents on out-of-domain tasks (e.g., training on a NLI task and evaluating on a math dataset), demonstrating the generalization capabilities of our approach. We also explore the collaboration among agents of varying capabilities, by analyzing the impact of training {\it heterogeneous} LLMs together.

\safevspace{-1.5mm}

To the best of our knowledge, this study is \textit{among the first works to explore the training of multi-LLM  systems as a whole}\footnote{Together with the contemporaneous works \citet{subramaniam2025multiagent} and \citet{zhao2025sirius},  both of which were released within the past month while preparing this paper.
In contrast to \ourstwo, the algorithms therein were based on (iterative) SFT, instead of RL. Also, \citet{motwani2024malt} provided a method to train verifier-generation-refiner system with DPO. },  
using RL, for multi-LLM collaboration.

\begin{figure*}[t]  %
  \centering
  \includegraphics[width=\textwidth]{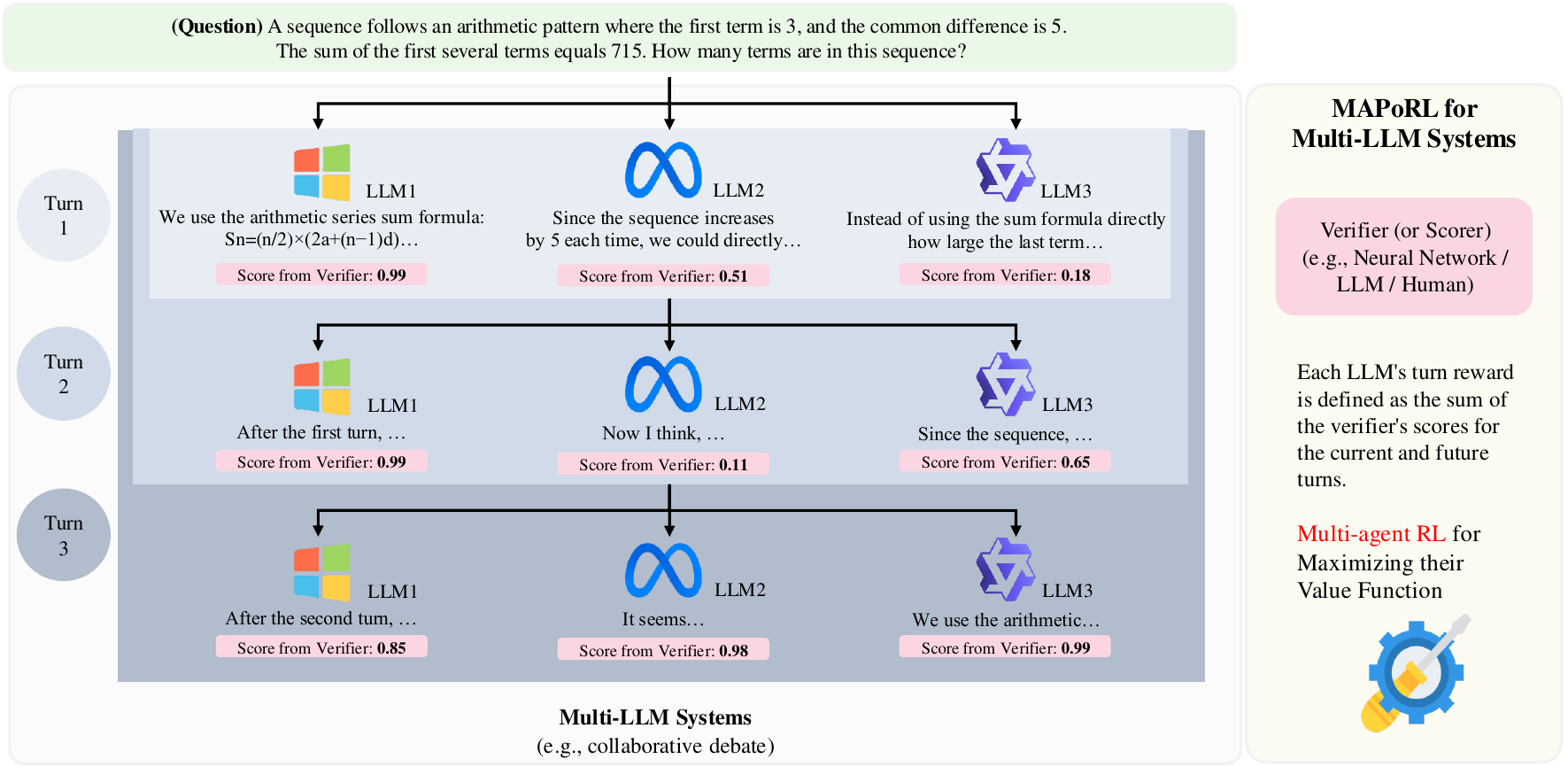}
  \vspace{-5mm}
  \caption{
  \oursspace can be applied to any multi-{LLM}  system with a scorer/verifier. In the illustrated example, it is integrated into a collaborative debate system for mathematical problem-solving. LLMs generate responses based on the multi-agent system pipeline, and a scorer/verifier evaluates their outputs. The reward for each LLM is determined based on these scores, which may include both current and future pipeline evaluations. Multi-Agent RL is employed to maximize each agent's value function.}
  \vspace{-0.35cm}
  \label{fig1:problem-setting}
\end{figure*}

\safevspace{-2.2mm}
\safevspace{-1.5mm}

\safevspace{-2.2mm}
\safevspace{-1.5mm}

\section{Analytical Insights: Collaborate  to Solve Hard Questions}
\safevspace{-2.2mm}
\label{sec:theory}
In this section, we present a simplified model of multi-LLM collaboration and explain (a) why {\it co-training}  multiple LLMs is necessary compared to training a single agent,  and (b) the role of incentives to further enhance collaboration during training. We validate both aspects through experiments in \Cref{sec:exp}.
\safevspace{-2.2mm}
\safevspace{-1.5mm}

\subsection{Problem Setup}
\safevspace{-2.2mm}

We consider questions that inherently require collaboration for a successful solution. For instance, solving complex mathematical problems often requires  collaboration  among  multiple agents \citep{liang2023encouraging, du2023improving}. Beyond mathematics, collaboration can also enhance the performance on tasks related to privacy, factuality, and reliability \citep{feng2025llmdroolsmultillmcollaboration}.
{We model the interaction among LLMs as a repeated game with $T$ turns}. 
For simplicity, we assume that in each turn, each agent chooses between two actions: \textit{Collaborate} ($a_0$) or \textit{Act Independently} ($a_1$). For a given question $q$, we define \(C(q)\) as a non-negative integer representing the \textit{collaboration threshold}. The agents achieve \textit{collaborative synergy} if, over the course of the $T$-turn interactions, the total number of collaborative actions (\(a_0\)) of all the agents meets or exceeds \(C(q)\).
When collaborative synergy is achieved, each agent receives a reward \(R_{\text{syn}}(q) = 1\), representing a (near-)guaranteed correct solution. Prior to achieving synergy, agents receive rewards based on their chosen actions: a reward of \(R_{\text{col}}(q)\) for choosing to collaborate \((a_0)\) and \(R_{\text{ind}}(q)\) for acting independently \((a_1)\), where \(R_{\text{col}}(q) < R_{\text{ind}}(q)\) (see \Cref{rmk:rationale} for a detailed  justification on the setup). This reward structure creates a tradeoff between short-term accuracy and long-term collaborative success. {This  setup is related to the classical Coordination Games  \cite{cooper1999coordination} in game theory if $R_{\text{syn}}$ is large}. We introduce a new collaboration threshold and synergy mechanism that shapes the transition from independent actions to collaborative behavior {in {\it multiple turns}}, to better model the collaboration procedure among multiple LLMs. 

\vspace{-3mm}
\begin{remark}[Rationale Behind the Setup]
\label{rmk:rationale}
\upshape
This formalization captures several key aspects of complex problem-solving dynamics. Choosing to collaborate ($a_0$) represents contributing \textit{exploratory ideas} or \textit{partial solutions}. While these contributions have a lower immediate probability of correctness \(R_{\text{col}}(q)\), they are essential building blocks towards the complete solution. Acting independently $(a_1)$ represents using conventional approaches that may yield a higher \textit{immediate probability} of correctness \(R_{\text{ind}}(q)\), but may contribute less to solving particularly challenging problems. The collaboration threshold \(C(q)\) represents the minimum amount of collaboration efforts and idea generation needed to solve complex problems. Once this threshold is reached (i.e., achieving collaborative synergy), the agents can combine their insights to solve the challenging problem, with a higher  reward \(R_{\text{syn}}(q)\).
\end{remark}
\vspace{-3mm}

\safevspace{-3mm}

\subsection{Analytical Observations}
\safevspace{-1.5mm}

To provide intuition for why co-training is necessary and single-agent training may be inadequate, 
we analyze the simplest case with $T=2$ and $C(q) = 1$ to illustrate the fundamental differences between single and multi-agent training. We provide formal statements and proofs in \Cref{appendix_defproof}.

\begin{restatable}{observation}{claimsingle}
\label{claim:1}
Suppose that the opponent selects action \( a_0 \) with probability \( \pi(q) \) for each question \( q \). Then, the optimal strategy for the agent is as follows: if  $(R_{\text{syn}}(q) - R_{\text{ind}}(q))\pi(q) \geq R_{\text{ind}}(q) - R_{\text{col}}(q)$, then the optimal strategy for question \( q \) is to collaborate (\( a_0 \)). Otherwise, the optimal strategy is to act independently (\( a_1 \)).
\end{restatable}

\safevspace{-2mm}

This shows the dependence of the agent's strategy on the opponent's behavior. If the opponent is {\it not collaborative} enough  and {\it non-strategic}, then \( \pi(q) \) will be small, leading the trained model to behave in a  \textit{non-collaborative} way.

\begin{observation}[Informal]  
\label{claim:2}
If both agents are trained to maximize their individual cumulative rewards with an entropy regularization term scaled by $\tau$, then as $\tau \to 0$, they will collaborate if:

\vspace{-6.2mm}
{\small
\[
R_{\text{syn}}(q) > \max( 3R_{\text{col}}(q) - 2R_{\text{ind}}(q), 2R_{\text{ind}}(q) - R_{\text{col}}(q)).
\]
}
\end{observation}

\safevspace{-1.9mm}
{\Cref{claim:2} can be proved by adapting the results of \citet{zhang2016quantal}, and transforming our setup with \(T=2\) into a matrix game.} This observation implies that when both agents optimize their own {cumulative} reward, they will naturally choose collaboration when \( R_{\text{syn}}(q) \) is high enough, which emphasizes the importance of additional incentives to promote collaborative synergy. Due to this observation, in \Cref{sec:incentive-coll}, we incentivize collaboration by providing a higher \( R_{\text{syn}}(q) \).

\safevspace{-2mm}

\subsection{Toy Experiments with \(T=10, 20\) Turns}
\label{ssec:exp-toy}
We illustrate the benefit of \textit{jointly optimized} (multi-agent) policies over {those obtained from a} \textit{single-agent} approach in our setting, with longer \(T=10, 20\) turns.  Each question \(q\) is associated with the  rewards \(R_{\text{col}}(q)\), \(R_{\text{ind}}(q)\), and \(R_{\text{syn}}(q)\), along with a collaboration threshold \(C(q)\).  Further details on the choices of these quantities be found in \Cref{appendix:toy}.

{We first consider a \textit{single} agent interacting with a fixed opponent whose probability of collaborating, \(\pi(q)\), is set at \(\{0.5,\,0.6,\,0.7\}\). Despite the relatively high likelihood of collaboration from the opponent, the single-agent policy, which optimizes its response to the fixed opponent, does not result in effective collaboration (\Cref{figs:collaboration_probability}). Instead of learning to strategically engage with the opponent’s behavior, the single-agent policy, which follows a best-response strategy to the fixed opponent, tends to \textit{avoid} collaboration, highlighting the limitations of a single-agent framework when facing a \emph{fixed}, \emph{non-strategic} opponent.}

Next, we consider two \textit{jointly optimizing} (multi-agent) learners who adapt their policy based on the other's actions. Concretely, we compute an  entropy-regularized Nash equilibrium for \(\tau=0.1\) via backward induction. 
As shown in \Cref{figs:collaboration_probability} (solid red curves), the \textit{jointly} optimized agents coordinate with significantly higher collaboration rates. Intuitively, this is because their learning process, which fosters high collaboration at larger $t$, shapes their strategic behavior from the very beginning, leading to increased cooperation even at the first turn. 
These toy results underscore the importance of \textit{strategic interactions}: when both agents adapt their policies simultaneously, they learn to be more collaborative,  despite the temptation of short-term independent rewards. 
\vspace{-2mm}

\section{{Post-Co-Training} with MARL}
\label{sec:MARL}
\vspace{-2mm}
\safevspace{-2mm}
{We now provide an overview of our new paradigm of \textit{Multi-Agent Post-Co-Training with RL (\ours)} for LLM collaboration. 
In our framework, each agent's response per turn is evaluated by a verifier, which assigns a score reflecting the answer's validity. The reward 
is defined as the (weighted) sum of the verifier's scores from the current turn and those from the future turns, thus capturing both the immediate feedback and the projected long-term impact of the agent's response. The agents’ policies are updated using multi-agent PPO for each turn, ensuring that the learning process incorporates both the performance {in the current turn} and the influence of the anticipated collaborative interactions {in the future turns}.} 

\vspace{-2mm}

\subsection{Multi-Agent System - Collaborative Debate Formulation}
{We follow the collaborative debate system proposed by \citet{du2023improving}as an example of our multi-LLM system in the experiments.}  Note that, \ourstwo~can be applied to other multi-LLM systems as long as each agent’s response can be evaluated—for example, by a verifier that assigns a score reflecting the quality or correctness of the response.  The reward for each agent is then determined by summing the verifier scores of all the responses influenced by that agent throughout the {multi-agent interaction} process.
Assume we have a collaborative debate system that runs for \( T \) turns and involves \( A \) agents.  In each turn, an LLM must determine its next response based on the history of its own answers as well as the answers provided by other LLM agents. Let \( q \) be the given question, and let \( s_{ti} \) denote the solution provided by agent \( i \) at turn \( t \).  We inductively express the solution \( s_{(t+1)i} \) as follows:
\safevspace{-0.3cm}

{
\vspace{-0.5cm}
\small
\begin{align*}
s_{1i} = \text{LLM}_i\left(q\right), \,
s_{(t+1)i} = \text{LLM}_i\left(q \oplus_{j \in [A], t' \in [t]} s_{t'j}\right)  
\end{align*}}
\safevspace{-8mm}
\vspace{-0.5cm}

\noindent where \( \oplus \) denotes token-wise concatenation, $1 \leq t \leq T-1$ and \( \text{LLM}_i(s) \) represents the function of inputting {prompt}  \( s \) into the \( \text{LLM}_i \) {which outputs logits over its token space, followed by sampling a token based on these logits.} If $A = 1$, then this setup is equivalent to that of  self-correcting LMs~\cite{madaan2024self}. Now, we define \(\btheta = (\theta_{ta})_{t\in[T], a \in [A]}\), where \(\theta_{ta}\) represents the parameters of the \(a\)-th agent at turn \(t\). We denote LLM parameterized by  $\theta_{ta}$ as \(LLM_{\theta_{ta}}\).

Next, to implement \ourstwo, we define the reward function for the multi-agent RL formulation,   using the verifier score. {Due to space constraints, we here introduce the \textit{Influence-aware Verification Rewards}, and defer other choices of the reward functions that we will use in the experiments to \Cref{appendix_reward}.} 

\vspace{-2mm}
\begin{definition}[Influence-aware Verification Reward]
\label{def:infl-veri}
The influence-aware verification reward function $R_{\btheta}(q, s_{ta})$ is defined as 
\begin{align*}
    &R_{\btheta}(q, s_{ta}) = \EE \bigg[\frac{1}{\sum_{t' \in [t, T]} \gamma^{t'-t}} \biggl( \text{Verifier}(q, s_{ta}) 
    \\
    &\, + \sum_{t' \in [t+1, T]}\sum_{j \in [A]} \frac{1}{A}\gamma^{t'-t}\text{Verifier}(q, s_{t'j})\biggr) \bigg].  
\end{align*}
\end{definition}

{Here, the expectation arises from the randomness of other agents' answers, which are influenced by the agent's current response, and \(\gamma \in [0,1]\) is a discount factor.}
This reward not only considers the verifier's score for the current solution \( s_{ta} \), but also incorporates the impact of this solution on the future answers of all agents. The term \( \sum_{j \in [A]} \frac{1}{A} \) averages the verifier's scores across all the agents, reflecting the influence that \( s_{ta} \) has on the overall multi-agent system.

\subsection{Multi-Agent RL Formulation}
The reward of each agent, as well as its answer generation, is intertwined with the actions of other agents in the multi-LLM system with the reward function (\Cref{def:infl-veri}). Thus, instead of single-agent RL, we design a multi-agent RL approach. For this paper, we choose multi-agent PPO \citep{yu2022surprising} as a representative multi-agent RL algorithm and instantiate it in the language domain. 

Conventional multi-agent PPO extends PPO to multi-agent settings by training decentralized policies, either with a shared critic or through independent learning, where each agent optimizes its policy based on local observations and rewards. Typically, agents learn independently or with limited coordination mechanisms, such as centralized critics or shared value functions. Our approach adapts multi-agent PPO by defining the state as the concatenation of the multi-agent interaction history, allowing agents to condition their responses on past interactions. Additionally, we introduce reward structures that are aligned but not fully identical across agents, encouraging them to fulfill different roles while collectively working toward solving the task.

Since we are solving multi-turn problems, the value function for each turn's state {needs to be defined}. The state of each turn's  value $V_{ta\btheta}$ is the expectation of the cumulated reward  conditioned on the input text $i_{ta}^x$, which is defined as 

\safevspace{-6mm}
{\small
\begin{align*}
    V_{ta\btheta}(i_{ta}^x) = \EE \left[\sum_{x' = x}^{\texttt{length}(s_{ta})} r_{\btheta} (q, s_{ta}^{1:x'}) \bigggiven q, i_{ta}^{x} \right].
\end{align*}}
\safevspace{-5mm}

\noindent Here, \( i_{ta}^{x'} = q \oplus_{t' \in [t-1], j \in [A]} s_{t'j} \oplus s^{1:{x'}}_{ta} \) and 

\safevspace{-4mm}

{
\small
\begin{align*}
    &r_{\btheta}(q, s_{ta}^{1:x'}) = \pmb{1}(x' = \texttt{length}(s_{ta})) R_{\btheta}(q, s_{ta}) 
\\
&\, -\lambda_{\text{KL}} \text{KL}\left(\text{LLM}_{\theta_{\text{ref},ta}}(i_{ta}^{x'}) \,\|\, \text{LLM}_{\theta_{ta}}(i_{ta}^{x'})\right), 
\end{align*}
}

\safevspace{-4mm}

\noindent where $t$ denotes the turn index and $a$ refers to agent  $a \in [A]$, $s_{ta}^{1:x}$ represents the generated token from agent $a$ up to the $x$-th token in turn $t$, with $\theta_{\text{ref},ta}$  denoting the parameter of a reference LLM, and  \( \lambda_{\text{KL}}\geq 0\) is some regularization coefficient. As per our reward construction, the value maximization 
not only considers the current turn’s verifier score,  but also anticipates future verifier scores from the same or other agents across multiple turns, which makes multi-agent training {relevant}. {We estimate the advantage function using Generalized Advantage Estimation (GAE)~\citep{schulman2015high}, which leverages the value function to measure how much better the current token selection is compared to the baseline value function.  

The value function is approximated by a neural network with parameter \(\theta_{vta}\), denoted as \(V(i_{ta}^x;\theta_{vta})\), which serves as an estimate of \(V_{ta\btheta}(i_{ta}^x)\). Using \(V(i_{ta}^x;\theta_{vta})\), we estimate the advantage function \(A(i_{ta}^x;\btheta, \theta_{vta})\) via GAE. } The loss function for multi-agent PPO is then given by: 
\vspace{-7mm}

\begin{align*}
    &L_{\text{PPO}}(\theta_{ta}, \theta_{vta})  = L_{\text{Surrogate}}(\theta_{ta}) + L_{\text{Value}}(\theta_{vta}),
\end{align*}
\vspace{-8mm}

\noindent where $L_{\text{Surrogate}}(\theta_{ta})$ is defined as {

\vspace{-5mm}

{\small \begin{align*}
\mathbb{E}\Bigl[
  &\min\Bigl(
    R_{ta}^x\,A\bigl(i_{ta}^x; \btheta_{\text{old}}, \theta_{\text{old},vta}\bigr), 
     \\
    &\qquad \qquad \texttt{clip}_{\epsilon}\bigl(R_{ta}^x\bigr)\,
    A\bigl(i_{ta}^x; \btheta_{\text{old}}, \theta_{\text{old},vta}\bigr)
  \Bigr)
\Bigr].
\end{align*}}
\safevspace{-7mm}

\noindent and $L_{\text{Value}}(\theta_{vta})$ is defined as 

\vspace{-5mm}

{\small 
\begin{align*}
    &L_{\text{Value}}(\theta_{vta}) = \EE \left[\lambda_{\text{value}} \left(V(i_{ta}^x; \theta_{vta}) - V^\text{target}_{ta}(i_{ta}^x)\right)^2\right].
\end{align*}
} 
\vspace{-5mm}

\noindent Here, $\texttt{clip}_{\epsilon}(\alpha) := \min(\max(1-\epsilon, \alpha), 1+\epsilon)$, \( R_{ta}^x = \frac{\text{LLM}_{\theta, ta}(s^{x+1}_{ta} \mid i_{ta}^{x})}{\text{LLM}_{\theta_{\text{old}, ta}}(s^{x+1}_{ta} \mid i_{ta}^{x})} \), \(\btheta_{\text{old}} = (\theta_{\text{old},{ta}})_{t\in[T], a \in [A]}\) is the parameter used in the rollout for multi-agent PPO, and 

\vspace{-5mm}
{\small \begin{align*}
    V_{ta}^{\text{target}}(i_{ta}^x) = V(i_{ta}^x; \theta_{\text{old}, vta}) + A(i_{ta}^x; \btheta_{\text{old}}, \theta_{\text{old}, vta})).
\end{align*}}
\noindent The expectation \( \mathbb{E} \) is taken over the randomness from \[q \sim \mathcal{Q}, s_{t'a'} \sim \text{LLM}_{\theta_{\text{old}, t'a'}}(q \oplus_{t'' \in [t'-1], j \in [A]} s_{t'j}) \] for all $t' \in [t], a \in [A]$, and $x \sim \text{Unif}([\texttt{length}(s_{ta})])$, where $\cQ$ denotes the distribution of questions.

Each agent for each turn optimizes its policy and value function simultaneously, over the parameters \((\theta_{ta}, \theta_{vta})\). These agent interactions among multiple LLMs inherently lead to a multi-agent RL problem, rather than a single-agent RL one, as each agent influences others' learning processes throughout training. 

\safevspace{-2mm}

\subsection{Reward {Shaping to Incentivize Collaboration}}
\label{sec:incentive-coll}
As discussed in Section~\ref{sec:theory}, incorporating {additional incentives in the reward}  can steer  agents towards better   collaboration. We define four key parameters when implementing such a reward-shaping: parameters $\alpha_0$ and $\alpha_1$ correspond to the incentives related to an agent's own revision of the answer, {and parameters $\beta_0$ and $\beta_1$ correspond to those related to her influence on other agents' answers.} Specifically, $\alpha_0$ represents the ability to extract useful information from incorrect answers ({\it critical reasoning}), while $\alpha_1$ reflects an agent's tendency to be {\it persuaded} by the correct information. Meanwhile, $\beta_0$ represents the ability to provide incorrect answers that still contain useful information, potentially leading to better responses in the future turns. In contrast, $\beta_1$ captures an agent's ability to effectively {\it persuade others} when providing correct answers. We provide \Cref{tab:bonus_penalty_case1} and \Cref{tab:bonus_penalty_case2} to summarize the design of these incentives. 

\begin{table}[h]
    \small
    \centering
    \resizebox{\linewidth}{!}{  %
    \begin{tabular}{cccc}
        \toprule
        Answer (t) & Answer (t+1) & Majority (t) & Incentive \\
        \midrule
        R & W & R & -$\alpha_1$ \\
        R & W & W & -$\alpha_0$ \\
        W & R & W & $\alpha_0$ \\
        W & R & R & $\alpha_1$ \\
        \bottomrule
    \end{tabular}
    }
    \vspace{-2mm}

     \caption{{The design of additional incentives} regarding an agent's own answer revision in \ourstwo. The incentive is determined by how an agent changes its answer between consecutive turns ($t$ and $t+1$) relative to the majority opinion of others. \textbf{R} indicates a correct answer, \textbf{W} indicates an incorrect answer. The incentive value is applied to the agent's answer at turn $t+1$.}
    \label{tab:bonus_penalty_case1}
        \vspace{-3mm}
\end{table}

\begin{table}[h]
    \small
    \centering
    \resizebox{\linewidth}{!}{  
    \begin{tabular}{cccc}
        \toprule
        Majority (t) & Majority (t+1) & Answer (t) & Incentive \\
        \midrule
        R & W & R & -$\beta_1$ \\
        R & W & W & -$\beta_0$ \\
        W & R & W & $\beta_0$ \\
        W & R & R & $\beta_1$ \\
        \bottomrule
    \end{tabular}
    }
    \vspace{-2mm}
    \caption{{The design of additional incentives} regarding an agent's influence on other agents' answers in \ourstwo. The incentive is based on how the majority opinion changes between consecutive turns $t$ and $t+1$ relative to the agent's answer at turn $t$. The incentive value is applied to the agent's answer at turn $t$.}
    \label{tab:bonus_penalty_case2}
            \vspace{-3mm}

\end{table}

\section{Experiments}
\label{sec:exp}

\safevspace{-1.5mm}
\subsection{Datasets}
\safevspace{-1mm}

We evaluate \oursspace on  two benchmark NLP tasks to validate its performance in both mathematical reasoning and logical natural language inference. The details are summarized as follows:
\safevspace{-2mm}

\paragraph{GSM8K \citep{cobbe2021training} and TinyGSM \citep{liu2023tinygsm}.} 
GSM8K is a benchmark dataset designed to assess a model's mathematical reasoning abilities, requiring models to solve high-school-level mathematics problems. TinyGSM is an augmented version of GSM8K, generated using GPT-3.5-turbo, where solutions are provided in Python. Importantly, we did not utilize the reasoning processes from GSM8K or TinyGSM but rely solely on their final answers. For training the verifier model, we used 7,463 samples from GSM8K. Additionally, we incorporated the first 12,800 synthetic samples from TinyGSM for \ours\footnote{We divided the dataset for training the verifier and training \ours~to prevent overfitting LLMs to the trained verifiers.}. For evaluation, we hold out 1,319 samples from GSM8K as a test set. 

\paragraph{Adversarial Natural Language Inference (ANLI)~\citep{nie2019adversarial}.}
ANLI is designed to evaluate a model’s natural language understanding by presenting adversarial examples that challenge even the state-of-the-art models. To train the verifier model, we used first 10,000 training examples. Furthermore, we used the next 12,800 examples for \ours~training and 1,200 samples for testing.

\paragraph{Evaluation Method.} 
After all turns are completed, the final answer is determined using a majority voting scheme among the agents' responses. The accuracy is based on whether the majority-selected response is correct.
In cases where no clear majority winner emerges (e.g., a tie in vote counts), we adopt an expectation-based approach by weighting the correctness of each tied response proportionally. For example, if two agents receive an equal number of votes, the final score is adjusted as the expected accuracy of selecting the first agent’s answer as the final result. This ensures a continuous evaluation metric rather than an arbitrary tiebreaker.

\subsection{Models}
We primarily use the Microsoft Phi-3-mini-128k-instruct (3.4B) model~\citep{abdin2024phi}, together with Qwen2.5-3B-instruct~\citep{yang2024qwen2} and Llama-3-8B-instruct~\citep{dubey2024llama} for  the experiments. Due to computational constraints, we mainly use quantized models and fine-tuned them with QLoRA~\citep{dettmers2024qlora}. We defer the training details to \Cref{appendix:training}. When evaluating on GSM8K and ANLI, we set the max token length to 300 and 250, respectively.

\subsection{Experiment 1: Vanilla Collaboration by Off-the-shelf LLMs Cannot Improve Performance, While \ours-Trained LLMs Can}
\label{sec-exp1}
We first compare the collaboration performance of off-the-shelf LLMs with \ours-trained LLMs. The training was conducted with two agents collaborating over three turns. An overview of the trained system is provided in \Cref{fig1:problem-setting}. In Experiment 1, we trained the model starting from turn \( t \geq 2 \) for two reasons: (a) the first turn primarily focuses on knowledge acquisition from each dataset, and (b) to ensure a fair comparison with off-the-shelf LLMs. We focus on enhancing collaboration skills rather than teaching specific task knowledge. For this experiment, we used Phi-3-mini-128k-instruct and evaluate the trained models in a three-agent and three-turn collaboration environment.

We observe that even when the off-the-shelf LLM is allowed to generate longer reasoning (600 tokens, twice the output length of our \ours-trained model  model), its accuracy did not improve across turns. 
This aligns with prior findings in the literature, particularly for models that are not sufficiently strong. For instance, \citet[Table 7]{huang2023large} provided  evidence that additional turns do not necessarily improve the  performance significantly. 
Similarly, our results show that off-the-shelf LLMs' performance may not benefit from additional turns. In contrast, LLMs trained using \ours~exhibit improved performance as the number of collaboration turns increased, as shown in \Cref{fig:turn-performance-gsm9k-anli}.

\begin{figure}[!h]
       \centering
        \includegraphics[width=\linewidth]{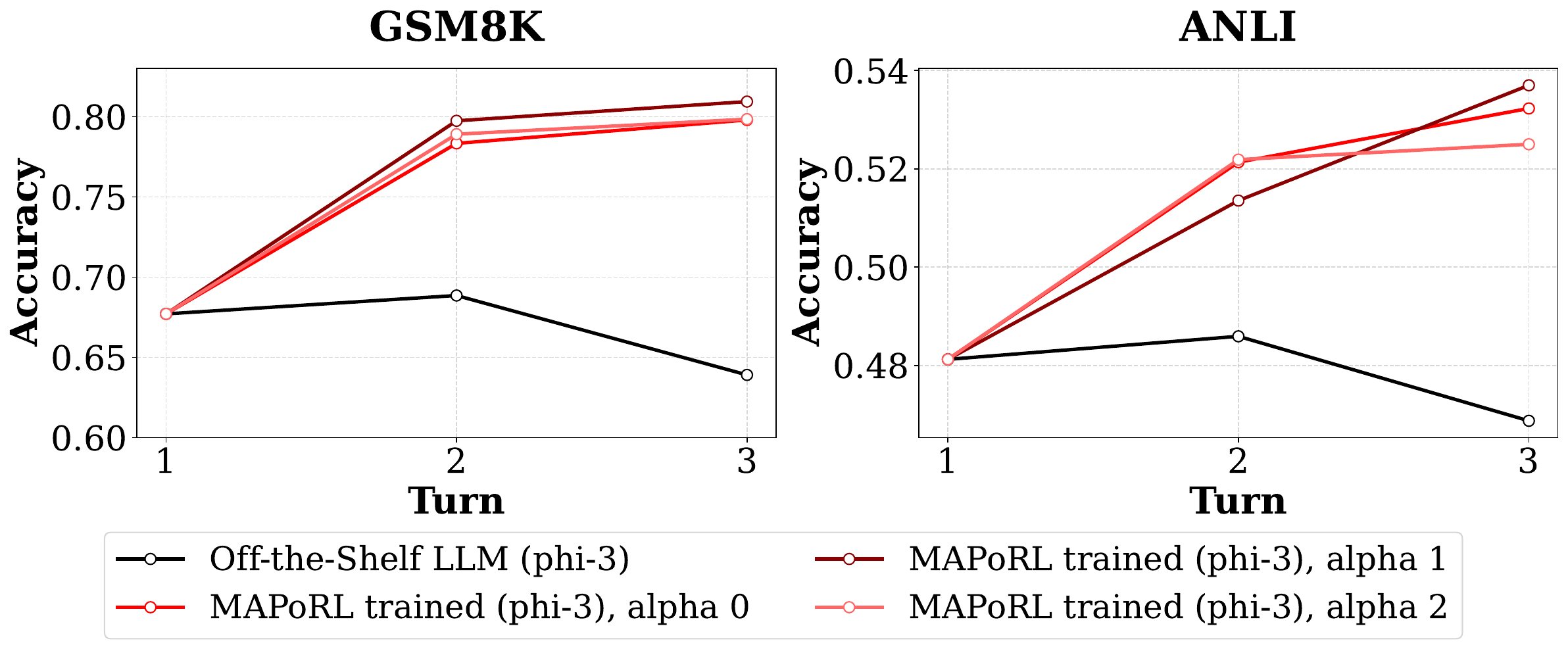}
    \safevspace{-4mm}
    \caption{Performance comparison of different LLMs across tasks (left: GSM8k, right: ANLI) under various settings. We evaluate collaboration ability in five conditions: (1) off-the-shelf LLMs collaborating and (2) models trained using \oursspace collaborating (with all incentive parameters (\Cref{sec:exp2}) $\alpha, \beta = 0,1, 2$, respectively).}
    \label{fig:turn-performance-gsm9k-anli}
\end{figure}
\safevspace{-2mm}

\begin{remark}[Domain-Specific Knowledge Acquisition vs. Collaboration Ability Improvement]
\upshape
One might question whether the performance gains observed in \ours-trained models stem from acquiring domain-specific knowledge rather than improved collaboration ability. To address this, we compare off-the-shelf LLMs and \ours-trained models by testing how well they perform on questions without any collaboration, providing \ours-trained models only the original question -- without interaction history -- to check if their performance is solely due to domain knowledge learned during training. The results are as follows:

\begin{table}[h!]
\small
\centering
\begin{tabular}{|c|c|c|c|}
\hline & \textbf{Phi-3} & \textbf{\oursspace T2} & \textbf{\oursspace T3} \\
\hline
GSM8k & 0.609 & 0.604 & 0.611 \\
\hline
ANLI & 0.451 & 0.458 & 0.453 \\
\hline
\end{tabular}
\end{table}
Here, we provide the same questions to the off-the-shelf Phi-3 model, the \ours-trained turn-2 model, and the \ours-trained turn-3 model. The similar performance across these models suggests that \oursspace training did not enhance task-specific knowledge but rather improved the models' ability to collaborate effectively. 
\end{remark}

\safevspace{-3mm}
We also provide the changes in the fraction of responses that transition their correctness over multiple turns of \ours. The fraction of Incorrect $\to$ Incorrect responses decreased, and the fraction of Correct $\to$ Incorrect responses also decreased, indicating that \oursspace enhanced effective collaboration.

\safevspace{-6mm}

\begin{figure}[!h]
    \centering
    \includegraphics[width=1\linewidth]{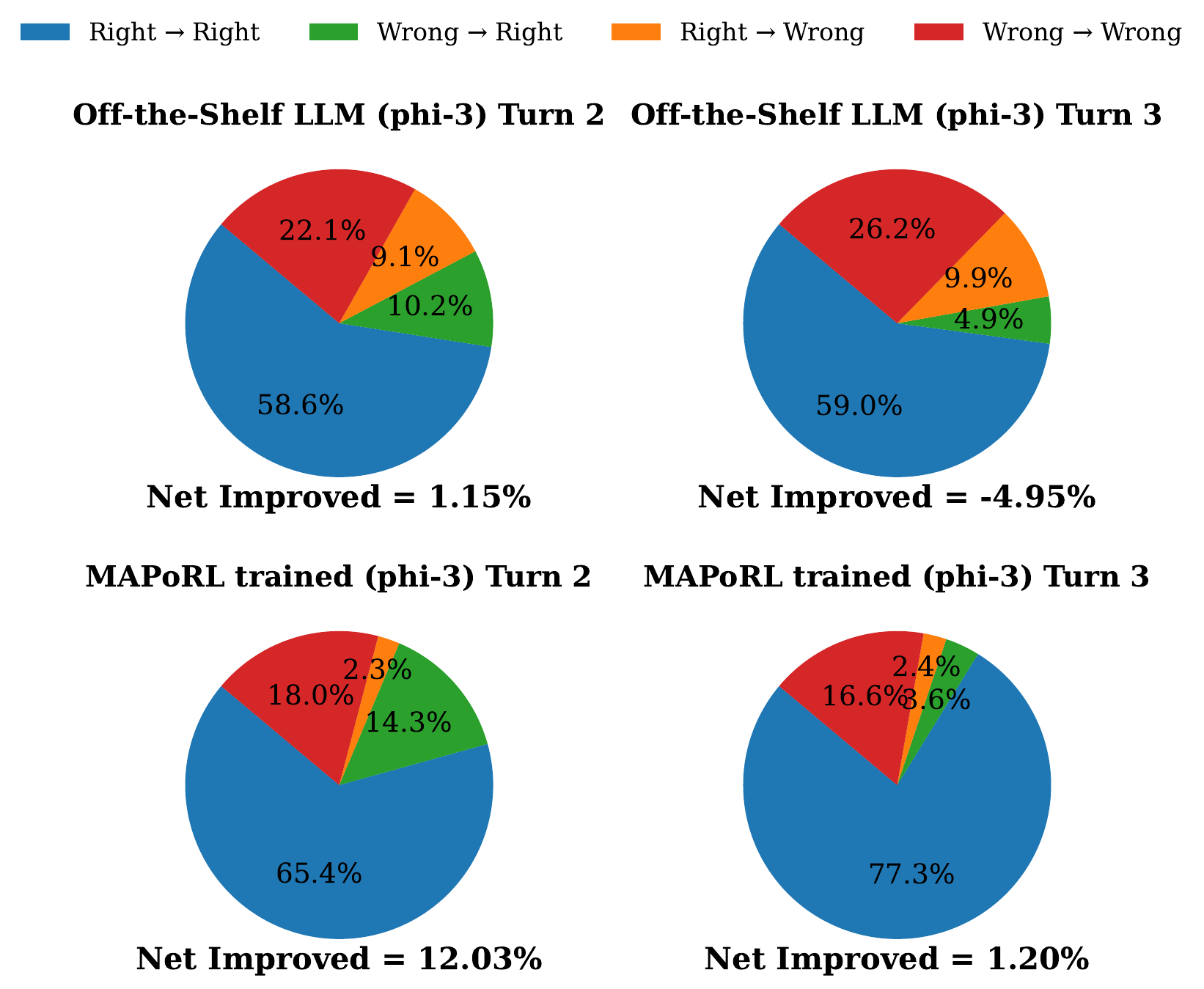}
\safevspace{-1mm}
    \caption{Changes in the fraction of responses that transition their correctness over multiple turns of \oursspace on GSM8k.} 
    \label{fig:chart-GSM}
\end{figure}
\safevspace{-6mm}

\begin{figure}[!h]
    \centering
    \includegraphics[width=1\linewidth]{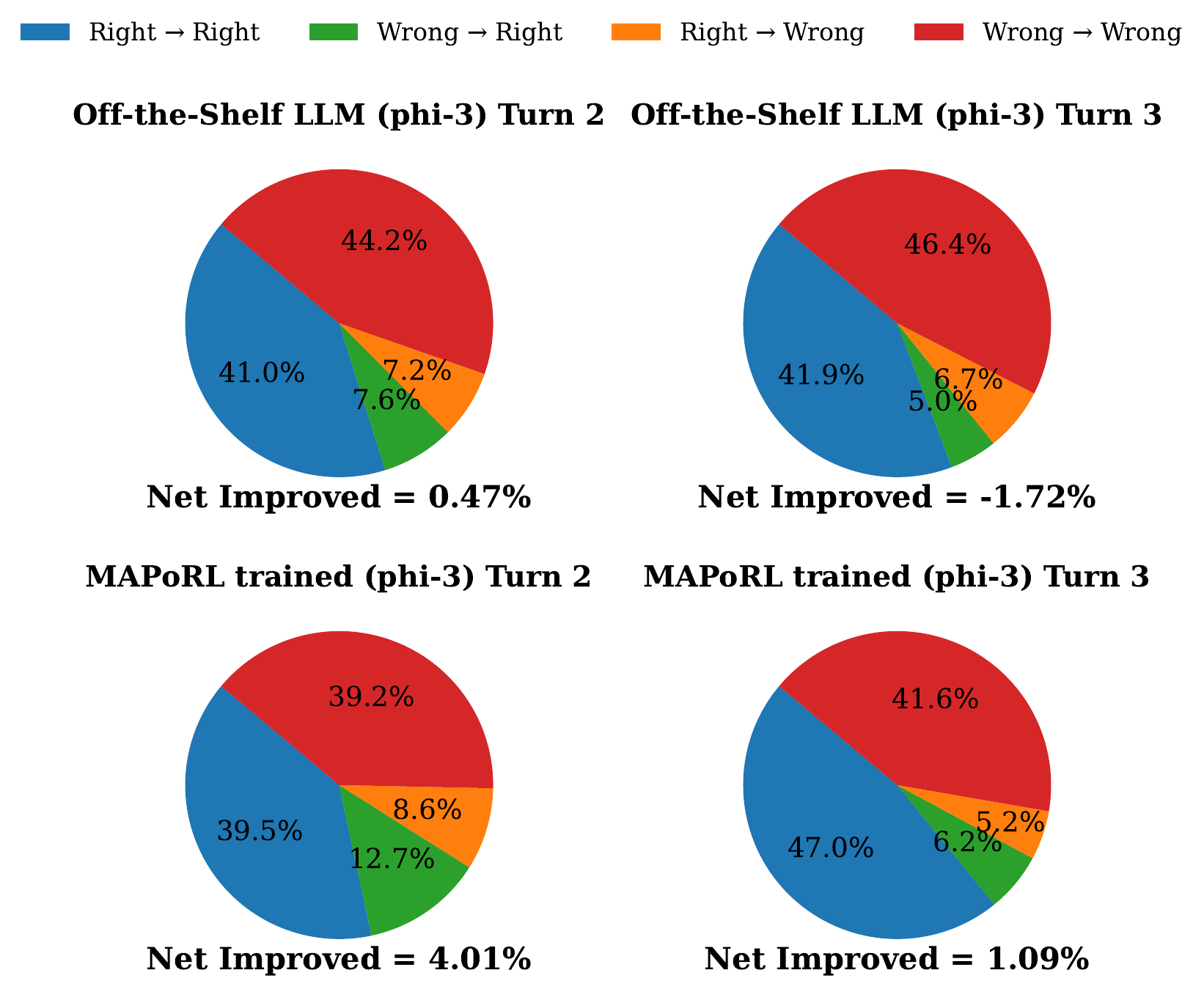}
    \caption{Changes in the fraction of responses that transition their correctness over multiple turns of \oursspace on ANLI.}
    \label{fig:chart-ANLI}
\end{figure}

\safevspace{-4mm}
\subsection{Experiment 2: Reward Shaping with Collaboration Incentives}
\label{sec:exp2}
In addition to the multi-agent independent PPO framework, we then investigate the auxiliary incentive mechanism designed to enhance collaborative interactions. 
To analyze the impact of the incentive parameters (\(\alpha\) and \(\beta\), \Cref{sec:incentive-coll}), we simplify our experimental setup by limiting the total number of debate turns to 2 and analyze the following cases. Here, \(\alpha_0\) and \(\alpha_1\) correspond to incentives for an agent’s own revision, capturing \textit{critical reasoning} (extracting useful information from incorrect answers) and \textit{persuadability} (accepting correct information), respectively. Meanwhile, \(\beta_0\) and \(\beta_1\) correspond to incentives for influencing others, where \(\beta_0\) encourages providing incorrect but useful responses, and \(\beta_1\) reflects an agent’s ability to \textit{persuade others} with correct answers.

To analyze the impact of the incentive parameters (\(\alpha\) and \(\beta\), \Cref{sec:incentive-coll}), we simplify our experimental setup by limiting the total number of debate turns to 2 and analyze the following cases. Here, \(\alpha_0\) and \(\alpha_1\) correspond to the incentives related to an agent's own revision of the answer, while \(\beta_0\) and \(\beta_1\) correspond to the incentives related to the agent's influence on other agents' answers.
\begin{table}[h]
    \centering
    \resizebox{\linewidth}{!}{  
    \begin{tabular}{ccccccc}
        \toprule
        ($\alpha_0, \alpha_1$) & RWR & RWW  & WRW & WRR & $\Delta_0$ & $\Delta_1$ \\ 
        \midrule
        (0, 0) &0.0529 & 0.0563 & 0.1244 & 0.2286 & 0.1757 &0.0661 \\
        (0, 2) & 0.0270 & 0.0521 & 0.1259 & 0.2194 & 0.1924 &0.0738 \\
        (2, 0) & 0.0500  & 0.0563 & 0.1241 & 0.2272 & 0.1772 &0.0678  \\
        \bottomrule
    \end{tabular}
    }
    \caption{Analysis of answer revision patterns under different $\alpha$ parameters. The columns RWR through WRR show the proportion of each transition type, where the three letters indicate Answer(t), Answer(t+1), and Majority(t) respectively. {R and W stand for right and wrong answer.} $\Delta_0$ measures the difference in transitions from wrong to right answers when the majority is wrong (WRW $-$ RWW) which is related to $\alpha_0$, while $\Delta_1$ measures transitions when the majority is right (WRR $-$ RWR) which is related to $\alpha_1$. }    \label{tab:alpha_analysis}
\end{table}
\safevspace{-3mm}

\paragraph{Analysis of $\alpha_0$ and $\alpha_1$.}
We compare baseline $(\alpha_0, \alpha_1) = (0,0)$ against two configurations: $(0,2)$ and $(2,0)$. When $\alpha_1$ was increased to 2, we observe a 9.5\% improvement in $\Delta_1$, indicating that incentivizing agents to follow correct majority opinions effectively improved performance. When $\alpha_0$ was increased to 2, we observed a smaller (2.57\%) improvement in $\Delta_0$, suggesting that rewarding agents for deviating from incorrect majority opinions had a positive but limited effect.
\begin{table}[h]
    \centering
    \resizebox{\linewidth}{!}{  
    \begin{tabular}{ccccccc}
        \toprule
        ($\beta_0, \beta_1$) & RWR & RWW  & WRW & WRR & $\Delta_0$ & $\Delta_1$ \\ 
        \midrule
        (0, 0) &0.0070 & 0.0453& 0.0226 & 0.0221 & 0.0151 &-0.0227\\
        (0, 2) & 0.0686 & 0.0461 & 0.0231 & 0.0230& 0.0161 &-0.0230 \\
        (2, 0) & 0.0011  & 0.0360 & 0.0161 & 0.0188& 0.0177&-0.0199 \\
        \bottomrule
    \end{tabular}
    }
    \caption{Analysis of majority opinion influence under different $\beta$ parameters. Meaning of the column is the same as Table~\ref{tab:alpha_analysis}.
    }    \label{tab:beta_analysis}
\end{table}

\vspace{-6mm}

\paragraph{Analysis of $\beta_0$ and $\beta_1$.}
We compare baseline $(\beta_0, \beta_1) = (0,0)$  against configurations $(0,2)$ and $(2,0)$. Increasing $\beta_1$ to 2 resulted in a slight decrease in $\Delta_1$ (-1.32\%), indicating that incentivizing agents based on their influence when correct did not improve outcomes. However, increasing $\beta_0$ to 2 lead to a substantial improvement in $\Delta_0$ (17.2\%), suggesting that rewarding agents for constructive influence even when wrong (providing useful incorrect answers that lead to better future responses) significantly enhanced collaborative performance.

{For a total debating turns of 3, we also plot the collaboration performance  using models trained with \(\alpha_i = \beta_i = 0, 1, 2\) for \(i = 1, 2\) on the GSM8K and ANLI tasks (\Cref{fig:turn-performance-gsm9k-anli}). The results showed some performance improvement, though the gain was relatively modest.}

\safevspace{-2mm}

\subsection{Experiment 3: Collaboration Ability Acquired by \ours~Is Transferable}
Here, we investigate the transferability of collaboration abilities acquired through \oursspace across different datasets not used during training. We evaluate LLMs trained with \oursspace on one dataset when applied to tasks from other datasets. For instance, we assesse models trained on ANLI when solving tasks from GSM8k, along with other dataset combinations.
The results, presented in Table \ref{table:transfer-stacked}, demonstrate that collaboration abilities learned through \ours~are indeed transferable across datasets. This suggests that the models acquire a \textit{meta}-capability for effective collaboration, even when encountering novel, unseen tasks.
\vspace{-3mm}

\begin{table}[!ht]
  \centering
  \resizebox{\columnwidth}{!}{%
    \begin{tabular}{llccc}
      \toprule
      \textbf{Training} $\rightarrow$ \textbf{Evaluation} & \textbf{Model} 
          & \textbf{Turn 1} & \textbf{Turn 2} & \textbf{Turn 3} \\
      \midrule
      \multirow{2}{*}{\textbf{ANLI} $\rightarrow$ \textbf{GSM8K}} 
         & Off-the-shelf & 0.677 & 0.688 & 0.640 \\
         & Trained       & 0.677 & 0.712 & 0.720 \\
      \midrule
      \multirow{2}{*}{\textbf{GSM8K} $\rightarrow$ \textbf{ANLI}} 
         & Off-the-shelf & 0.482 & 0.486 & 0.468 \\
         & Trained       & 0.482 & 0.499 & 0.507 \\
      \bottomrule
    \end{tabular}
  }
  \caption{Performance comparison (Accuracy) of 3-agent collaboration using 
           \textit{off-the-shelf} vs.\ \textit{trained} LLMs. 
           For each dataset pair (rows in bold), the first row shows 
           the off-the-shelf performance and the second row shows the 
           trained model performance, across Turns 1--3.}
  \label{table:transfer-stacked}
\end{table}

These findings demonstrate that models trained through \ours~on one task can effectively generalize their collaborative capabilities to different, unrelated tasks. This generalization ability suggests that \ours~develops fundamental collaborative skills that transcend specific task domains.

\subsection{{Experiment 4: \oursspace with Heterogeneous LLMs Can Help}}
\label{sec:exp4}
In this experiment, we investigate collaborative learning between different foundation models, specifically examining co-training between (Phi3 3.4B and Qwen2.5 3B) and (Phi3 3.4B and Llama3-8B) pairs. In single-model evaluations, both Phi3 and Qwen2.5 3B demonstrate stronger performance compared to Llama3-8B. Due to GPU memory constraints necessitating simultaneous loading of two base models, we conduct experiments in a two-agent, two-turn environment. This setup enables us to explore whether models with heterogeneous capabilities could effectively collaborate to enhance the overall performance (\Cref{fig:turn-performance-diff}). The synergistic effects are particularly evident when models with different strengths worked together, suggesting that diverse model partnerships can yield better outcomes than individual model performance alone when we have \ours.

\subsection{Experiment 5: Naïve Supervised Fine-Tuning Using High-Quality Collaboration Samples May Not Induce Collaborative Behaviors} 
\label{sec:exp3}
In this experiment, we investigate whether models could learn collaborative behavior through SFT on high-quality debate trajectories. We generated 12,800 trajectories using the multi-agent system (\Cref{fig1:problem-setting}) with off-the-shelf LLMs to match the training sample size used in \oursspace for GSM8K. To provide favorable conditions for SFT, we allow a maximum of 600 tokens per response, which exceeded the token limit used in our \oursspace experiments.
We selected the top 10\% of trajectories using the following criteria: 1) excluding trajectories without well-formatted answers, 2) filtering out trajectories where the final majority voting result was incorrect, and 3) selecting 1,280 trajectories based on the verifier's score of the final answer, which evaluates both correctness and reasoning quality. Interestingly, the results indicate that SFT not only failed to enhance collaborative behaviors,  but also led to a decline in performance compared to the off-the-shelf model. Specifically, for turn-2, accuracy dropped to 0.578 (\(\Delta = -0.111\)), and for turn-3, it further decreased to 0.525 (\(\Delta = -0.114\))\footnote{Initially, these unexpected results led us to validate our findings through multiple experiments with varying temperatures  for language generation. The consistent performance degradation across turns was observed in all the cases. This pattern suggests fundamental challenges in using SFT to maintain collaborative performance across multiple debate turns.}. 
 This suggests that either substantially more training data would be required to learn effective collaborative behaviors, or that SFT might not be an effective approach for inducing such behaviors. {Contemporaneously, \citet{subramaniam2025multiagent} and \citet{zhao2025sirius} enhance multi-agent performance by incorporating new techniques into iterative SFT with their own data augmentation to generate effective collaboration examples, demonstrating its potential when combined with additional refinements. In contrast, our approach does not leverage data augmentation, but uses RL.}

\section{Concluding Remarks, Limitations, and Potential Risks}

In this paper, we have introduced \ours, a new post-training paradigm that leverages multi-agent RL to explicitly foster the collaboration among multiple LLMs. Unlike methods that rely solely on prompting or single-agent fine-tuning, \oursspace focuses on \textit{co-training} multiple LLMs, ensuring that each agent adapts its policy not just to immediate feedback, but also to the strategic behaviors of other agents over multiple interactive turns. By incorporating a verifier network for reward shaping with incentives, the framework guides each agent’s responses that account for both short-term correctness and long-term collaborative potential, thus promoting collaborative discussions that lead to more accurate final answers. 

Through an extensive set of experiments on reasoning-intensive tasks -- such as GSM8K for mathematical problem-solving and ANLI for logical natural language inference -- our results demonstrate that off-the-shelf LLMs often do not improve the overall performance with additional debate turns. In contrast, \ours-trained agents show significant improvements with accuracy increasing as collaboration progresses. Crucially, these collaborative abilities are shown transferable across tasks, suggesting that once LLMs learn to collaborate, they can retain a generalizable ``collaboration skill'' applicable to different domains.
Furthermore, our experiments with heterogeneous LLMs highlight that \oursspace can also foster collaborative synergy even among models of varying capabilities.

\section*{Limitations}
Since we use instruction prompts as inputs to the LLMs, the output can vary significantly depending on the prompts. As the first methodology paper, our experiments are conducted on relatively small LLMs (3B to 8B parameters) for fast iteration, and the observed behaviors may differ on larger models.  
After all turns of multi-LLM interactions, we apply majority voting to determine the final answer. Using alternative mechanisms, such as a manager agent that makes the final prediction based on the responses from multiple agents, may further improve the overall performance. 

\section*{Potential Risks}
As our proposed approach encourages and facilitates collaboration among multiple LLM agents, when adversarial or malicious agents exist, our method could lead to unintended harmful outcomes by enabling their collaboration with others.

\section*{Acknowledgment} We sincerely thank \href{https://scholar.google.com/citations?hl=en&user=K5WXLYEAAAAJ&view_op=list_works&authuser=2&sortby=pubdate}{Jisu Jang} for their valuable feedback on \Cref{fig1:problem-setting}. K.Z. acknowledges the support from the U.S. Army Research Office grant W911NF-24-1-0085 and the NSF CAREER Award 2443704.
\bibliography{main}

\onecolumn
\appendix
\section{Detailed Related Work Discussion}
\label{appendix:literature}
\paragraph{Multi-Agent Reinforcement Learning. }
Various algorithms have been proposed to address multi-agent reinforcement learning (MARL) \cite{hernandez2019survey,zhang2021multi},  including multi-agent Proximal Policy Optimization (PPO) \citep{yu2022surprising}, and value function factorization techniques such as QMIX and VDN \citep{rashid2020monotonic, sunehag2017value}. In the context of language models and collaborative debating we focus on, MARL takes on a particular and unique form. Here, each agent's state is represented by the sequence of previous responses from all the agents, with each agent deciding the next token based on this history. LLMs provide compact state representations through their hidden layers, enabling the use of long debate histories. 

\paragraph{Multi-Agent Collaboration with LLMs. }
An array of studies have explored effective collaboration frameworks among multiple large language model agents to solve complex tasks~\citep{wu2023autogen, li2024improving, zhao2024longagent}. For example, ``role-playing''-based approaches utilized multi-agent LLMs by assigning a specific role to each LLM~\citep{li2023camel}, and ``multi-agent debate''-based approaches prompted  each LLM agent to solve the task independently and then discuss \citep{du2023improving,khandebating}. 
In a debate, the agents reason through each other's answers to converge on a consensus  response, which may improve the factual accuracy, mathematical ability, and reasoning capabilities of the LLM \citep{du2023improving,liang2023encouraging, kim2024adaptive}. Similar mult-agentic frameworks include voting \citep{wang2023selfconsistency}, group discussions \citep{chen2024reconcile}, and negotiating \citep{fu2023improving}. However, all of these frameworks rely heavily on prompt engineering, which may lead to sub-optimal results \citep{huang2023large}, and do not consider {\it training} LLMs specifically for collaboration. Therefore, while 
multi-LLM systems seem promising at the first glance, their performance may be limited when using the  out-of-the-box  (pretrained)  LLM with only prompt tuning, which highlights the need for {\it training} for better multi-agent collaboration. Recently, \citet{stengel2024teaching} introduced a training framework for accepting or rejecting persuasion in multi-agent systems. Additionally, {very recently,} 
\citet{subramaniam2025multiagent} and \citet{zhao2025sirius} focused on training the entire multi-agent systems using iterative SFT.
In contrast,  \ourstwo~employs (multi-agent) RL to train the whole multi-LLM system. Recently, after \ours was released, \citet{liao2025marft} provided a similar training system of multi-agents with reinforcement learning.

\paragraph{RL for LLM Training. }
RL has been widely used in post-training  LLMs, e.g., for improving  factuality~\cite{tianfine}, code generation~\cite{le2022coderl}, and more recently and significantly, reasoning \cite{guo2025deepseek}. One prevalent approach of RL for LLM training is RL from human feedback (RLHF) 
\citep{ziegler2019fine, ouyang2022training,bai2022training, ahmadian-etal-2024-back}. 
RL offers a smooth generalization to the {\it multi-turn} setting  based on the Markov decision process (MDP) model, and there have been attempts to apply multi-turn RL for LLM training, such as RLHF for multi-turn model training to enhance the dialogue abilities~\citep{shani2024multi}, or deriving multi-turn RL objective for the improvement of mathematical reasoning~\cite{xiong2024building}. However, the major difference from  our work is that, these works did not consider multi-agent settings for collaboration. Recently, \citet{kumar2024training} enhanced LLMs' ability to self-correct using an RL-based approach. Our framework can accommodate this case by using a single agent in  \ourstwo.

\section{Additional Literature Review}
\label{appendix:deferred_literature}
\paragraph{Multi-Agent RL. }
Multi-agent reinforcement learning (MARL) has achieved significant advancements, particularly in cooperative games and their real-world applications, such as coordinating robot swarms \citep{huttenrauch2017guided} and self-driving vehicles \citep{shalev2016safe}. (A comprehensive overview of MARL can be found in \citet{zhang2021multi}).  The primary challenge in MARL lies in the exponentially large action space, making it difficult to optimize the policy for each agent. Various approaches have been proposed to address this issue, including multi-agent Proximal Policy Optimization (PPO) \citep{yu2022surprising}, value function factorization methods (QMIX, VDN) \citep{rashid2020monotonic, sunehag2017value}, and network-based formulations for multi-agent learning \citep{park2024multi}. These methods aim to make MARL more scalable with a large number of agents, mostly focusing on the classical models of stochastic/Markov games.

In the context of language models and collaborative debate systems, MARL takes on a unique form. Here, each agent's state is represented by the sequence of previous responses from all agents, with each agent deciding the next token based on this history. The detailed mathematical formulation for reinforcement learning in language models can be found in several theoretical and empirical studies on reinforcement learning with human feedback (RLHF) (\textit{e.g.,} \citet{ouyang2022training, zhu2023principled, park2024principled}). LLMs provide high-quality state representations through their hidden layers, enabling the consideration of long debate histories. Moreover, the sequential nature of these interactions inherently captures non-Markovian policies due to the extended sequence of responses. 

\paragraph{Teaching LLM Self-Correction.}
As mentioned in the main paper, single-agent self-correction and multi-agent collaboration has a very interesting relationship. Single-agent self-correction and multi-agent collaboration rely on multi-turn interactions—either internally, within a single agent, or collaboratively, among multiple agents—to improve results by challenging initial outputs and refining them through iteration. In single-agent systems, self-correction functions like an internal debate. The agent evaluates its own output over multiple turns, identifying potential mistakes and proposing alternative solutions. This process mirrors human reflection, where reconsideration often leads to improved conclusions. Meanwhile, in multi-agent systems, different agents engage in a collaborative debate, questioning and refining each other’s answers. By interacting in multiple rounds, these agents combine their individual perspectives to correct errors and arrive at more accurate solutions. 

There are several prior works aiming to improve LLMs' ability to self-correct. First line of work is using prompting technique, which guides LMs via prompting to iteratively correct the model outputs~\citep{madaan2024self}. However, some works use the ground-truth labels to determine when to stop the self-correction~\citep{kim2024language, shinn2024reflexion, yao2022react}, which is not applicable in the real-world scenarios where answer is not available for the tasks, and it is shown that under such scenarios the models can not do self-correct effectively~\citep{huang2023large}.

Another line of works train LLMs to \textit{learn} self-correction; \citet{qu2024recursive} introduced an approach using stronger LLMs to obtain multi-turn trajectories that have better responses through the iteration, and uses this data to fine-tune LLMs to learn self-correction. Different from this work, our approach do not require stronger LLMs for demonstrations, relying solely on the reward for training. \citet{welleckgenerating} proposed supervised fine-tuning to train a corrector model that can edit the model response iteratively, but this is specified the type of collaboration in the  generate-then-refine pattern, which can be sub-optimal to learned by the models. \citet{kumar2024training} employed an RL-based approach for the self-improvement of language models.

\paragraph{Multi-Agent LLMs with Game Theory.} 
Recent work has actively explored the strategic interactions of LLM agents within game-theoretic frameworks, as demonstrated in studies such as  \citet{park2024llm,brookins2023playing, akata2023playing, lore2023strategic, fan2023can}. Our paper can be viewed as training LLMs as solvers of cooperative games such as solving mathematical problems together.

\section{Deferred Content of \Cref{sec:theory}}
\label{appendix_defproof}

\begin{remark}[Rationale Behind the Setup]
\label{rmk:rationale}
\upshape
This formalization captures several key aspects of complex problem-solving dynamics. Choosing to collaborate ($a_0$) represents contributing \textit{exploratory ideas} or \textit{partial solutions}. While these contributions have a lower immediate probability of correctness \(R_{\text{col}}(q)\), they are essential building blocks towards the complete solution. Acting independently $(a_1)$ represents using conventional approaches that may yield a higher \textit{immediate probability} of correctness \(R_{\text{ind}}(q)\), but may contribute less to solving particularly challenging problems. The collaboration threshold \(C(q)\) represents the minimum amount of collaboration efforts and idea generation needed to solve complex problems. Once this threshold is reached (i.e., achieving collaborative synergy), the agents can combine their insights to solve the challenging problem, with a higher  reward \(R_{\text{syn}}(q)\).
\end{remark}

\claimsingle*
\begin{proof}
For the last turn (\( {t} = 2 \)), regardless of whether the opponent selects \( a_0 \) or not, choosing \( a_1 \) is {an} optimal strategy. This is because:
\begin{itemize}
    \item If collaborative synergy has been  achieved, the agent will always receive \( R_{\text{syn}}(q) \) regardless of their action in the second turn.
    \item If collaborative synergy has not been achieved, since we know that \( R_{\text{col}}(q) < R_{\text{ind}}(q) \), the optimal choice is to select \( a_1 \) in the final turn to maximize the immediate reward.
\end{itemize}

Therefore, considering the cumulative reward for the turn $t=1$, the 
reward matrix is given as follows: 

\[
\begin{array}{c|cc}
 & \text{\( a_0 \) (Collaborate)} & \text{\( a_1 \) (Act independently)} \\ \hline
\text{\( a_0 \) (Collaborate)} & (R_{\text{col}}(q) + R_{\text{syn}}(q), R_{\text{col}}(q) + R_{\text{syn}}(q)) & (R_{\text{col}}(q) + R_{\text{ind}}(q), 2R_{\text{ind}}(q)) \\
\text{\( a_1 \) (Act independently)} & (2R_{\text{ind}}(q), R_{\text{col}}(q) + R_{\text{ind}}(q)) & (2R_{\text{ind}}(q), 2R_{\text{ind}}(q)) \\
\end{array}
\]

Since the opponent chooses \( a_0 \) with probability \( \pi(q) \), the expected reward for choosing \( a_0 \) is:

\[
(R_{\text{col}}(q) + R_{\text{syn}}(q)) \pi(q) + (R_{\text{col}}(q) + R_{\text{ind}}(q))(1-\pi(q)).
\]

The expected reward for choosing \( a_1 \) is $2R_{\text{ind}}(q)$. To determine the optimal strategy, we compare these two expected rewards. The agent should collaborate (\( a_0 \)) if:

\[
(R_{\text{col}}(q) + R_{\text{syn}}(q)) \pi(q) + (R_{\text{col}}(q) + R_{\text{ind}}(q))(1-\pi(q)) \geq  2R_{\text{ind}}(q).
\]
which is equivalent to 
\[
(R_{\text{syn}}(q) - R_{\text{ind}}(q))\pi(q) \geq R_{\text{ind}}(q) - R_{\text{col}}(q).
\]

Thus, if \( (R_{\text{syn}}(q) - R_{\text{ind}}(q))\pi(q) \geq R_{\text{ind}}(q) - R_{\text{col}}(q) \), the optimal strategy is to \textit{collaborate} (\( a_0 \)). Otherwise, the agent should act independently (\( a_1 \)) to maximize their cumulative expected reward.
\end{proof}

Now, we provide the formal statement of \Cref{claim:2}. Before doing so, we define the regularized Nash Equilibrium (NE).
\begin{definition}[Regularized NE]
An entropy-regularized Nash equilibrium is defined as a strategy profile $\pi^\ast$ where each player maximizes a regularized objective that combines the expected reward with an entropy term. Specifically, for each player \(i\), the equilibrium strategy \(\pi_i^\star\) satisfies 
\[
\pi_i^\star = \arg\max_{\pi_i} \; \mathbb{E}_{a_i \sim \pi_i,\, a_{-i} \sim \pi_{-i}^\star}\bigl[u_i(a_i,a_{-i})\bigr] + \tau H(\pi_i),
\]
where \(\tau>0\) is a temperature parameter and \(H(\pi_i) = -\sum_{a_i} \pi_i(a_i) \log \pi_i(a_i)\) is the Shannon entropy of the strategy, and $u_i$ is the utility function of player $i$. This entropy term smoothens the best response, leading to a softmax (or logit) formula of the optimal strategy:
\[
\pi_i^\star(a_i) = \frac{\exp\Bigl(\frac{1}{\tau}\, \mathbb{E}_{a_{-i} \sim \pi_{-i}^\star}\bigl[u_i(a_i,a_{-i})\bigr]\Bigr)}{\sum_{a_i'} \exp\Bigl(\frac{1}{\tau}\, \mathbb{E}_{a_{-i} \sim \pi_{-i}^\star}\bigl[u_i(a_i',a_{-i})\bigr]\Bigr)}.
\]
\end{definition}

\noindent \textbf{Observation 2. }
Consider a game where each agent maximizes their expected cumulative utility plus an entropy regularizer with a small regularization coefficient \(\tau > 0\). Let \(\text{NE}(\tau)\) denote the unique Nash equilibrium of the regularized game for a fixed \(\tau > 0\). As \(\tau \to 0\), the sequence of equilibria \(\text{NE}(\tau)\) converges to  Collaborate $(a_0)$ if   
$$R_{\text{syn}}(q) = 1 > \max( 3R_{\text{col}}(q) - 2R_{\text{ind}}(q),2R_{\text{ind}}(q) - R_{\text{col}}(q)).$$

\begin{proof}
Following the reasoning in showing  \Cref{claim:1}, we analyze the cumulative reward for the turn $t=1$. The reward matrix is given by:

\[
\begin{array}{c|cc}
 & a_0 \text{ (Collaborate)} & a_1 \text{ (Act independently)} \\ \hline
a_0 \text{ (Collaborate)} & (R_{\text{col}}(q) + R_{\text{syn}}(q), R_{\text{col}}(q) + R_{\text{syn}}(q)) & (R_{\text{col}}(q) + R_{\text{ind}}(q), 2R_{\text{ind}}(q)) \\
a_1 \text{ (Act independently)} & (2R_{\text{ind}}(q), R_{\text{col}}(q) + R_{\text{ind}}(q)) & (2R_{\text{ind}}(q), 2R_{\text{ind}}(q)) \\
\end{array}.
\]
If $R_{\text{syn}}(q) = 1 > 2R_{\text{ind}}(q) - R_{\text{col}}(q),$ then this game is a coordination game, and according to \citet[Theorem 1]{zhang2016quantal},
as \(\tau \to 0\), the regularized NE converges to the risk-dominant strategy \citep{harsanyi1988general} in a \(2 \times 2\) game. In this setting, {by definition,} the collaboration strategy \((a_0, a_0)\) is risk-dominant  \citep{harsanyi1988general} if:
\[
(R_{\text{col}}(q) + R_{\text{syn}}(q)) + (R_{\text{col}}(q) + R_{\text{ind}}(q)) > (2R_{\text{ind}}(q) + 2R_{\text{ind}}(q)),
\]
which is equivalent to 
\[
R_{\text{syn}}(q)  > 3R_{\text{ind}}(q) - 2R_{\text{col}}(q).
\]
Combining the two conditions completes the proof. 
\end{proof}

\section{Deferred Details in \Cref{ssec:exp-toy}}
\label{appendix:toy}

The game is solved using backward induction with the state represented as \((\text{turn}, \texttt{count})\), where \(\texttt{count}\) denotes the number of times \((a_0, a_0)\) has occurred in the history of the interactions. Both players choose actions to maximize their expected cumulative utility plus an entropy term times a coefficient \(\tau = 0.1\).

\paragraph{Choices of $R_{\text{col}}(q), R_{\text{ind}}(q), R_{\text{syn}}(q), C(q)$. } Each instance of a question \(q\) is associated with parameters drawn as follows: the independent action reward \( R_{\text{ind}}(q) \) is sampled from a uniform distribution \( R_{\text{ind}}(q) \sim \text{Unif}(0, 1) \). The collaborative action reward \( R_{\text{col}}(q) \) is then sampled condition on \( R_{\text{ind}}(q) \), following \( R_{\text{col}}(q) \sim \text{Unif}\Big(0, R_{\text{ind}}(q)\Big) \). The synergy reward is fixed as \( R_{\text{syn}}(q) = 1 \).

\begin{figure}[htbp]
  \centering
  \includegraphics[width=\textwidth]{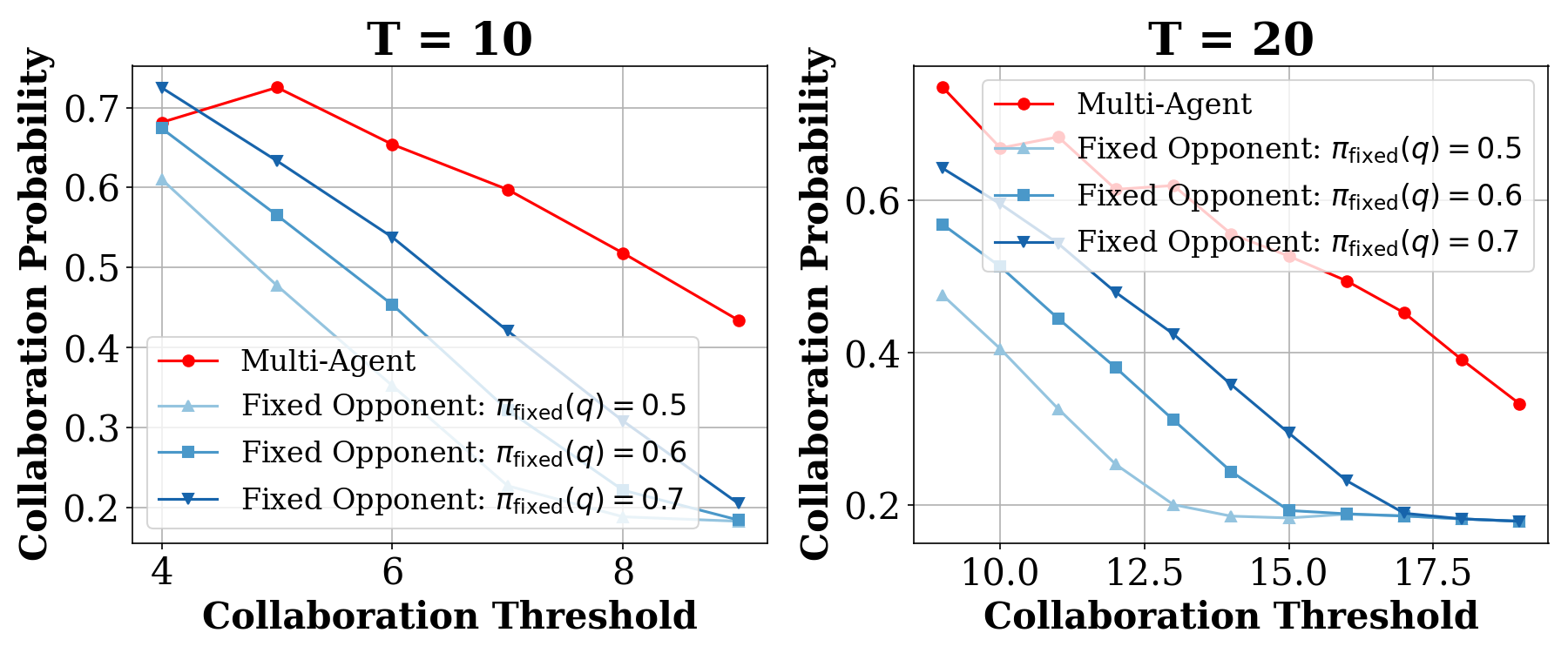}
    \safevspace{-5mm}
  \caption{
Collaboration probability (turn 1) as a function of the threshold \(C\), for two different horizons \(T=10\) (left) and \(T=20\) (right). 
We set the synergy reward to \(R_{\mathrm{syn}}=1\) and vary \(C\) from \(T{-}1\) down to \(\lfloor (T{-}1)/2 \rfloor\).  
The red curve (``Multi-Agent'') represents the collaboration probability when both players adaptively learn in a multi-agent setting.  
The blue curves show the best-response probabilities of Player 1 when facing a fixed opponent with collaboration probabilities \(\pi_{\mathrm{fixed}}(q) \in \{0.5,0.6,0.7\}\).  
Each data point represents an average over 5000 random samples of \((R_{\mathrm{ind}}, R_{\mathrm{col}})\). }
  \safevspace{-5mm}
  \label{figs:collaboration_probability}
\end{figure}

\section{Deferred Details of the  Verifier Models}
\label{appendix:verifier}
For a reasoning question $q$, the trained verifiers (reward models) assess the correctness of a complete solution path $s$, denoted as \(p(s \text{ is correct} \mid q)\) \citep{cobbe2021training, uesato2022solving, lightman2023let}. These reward models can either focus on the final outcome (outcome reward models) or provide step-by-step evaluations (process reward models). Although the latter generally yields better performance \citep{lightman2023let}, the limited availability of process-level annotated datasets—especially for challenging benchmarks like ANLI \citep{nie2019adversarial}—restricts its applicability. Additionally, while generating detailed trajectories for process supervision (as seen in \citet{wang2024math}) can be effective, our primary goal is not to enhance the language model's domain specificity. Consequently, we chose to adopt a simpler strategy by training a verifier based on a well-tuned output reward model. 

\paragraph{Verifier Models Structure.}
We used a quantized version of a language model as the backbone for the verifier. Additionally, we incorporated a linear head layer followed by a softmax layer to ensure that the verifier’s output falls within the range of 0 to 1. The default backbone model is Microsoft Phi-3-mini-128k-instruct \citep{abdin2024phi}. In experiments involving different model training setups (see \Cref{sec:exp4}), we employed a new verifier with a different base model, specifically the one used in \Cref{sec:exp4}. In these cases, we utilized Qwen2.5-3B-instruct \citep{yang2024qwen2} and Llama-3-8B-instruct \citep{dubey2024llama} as alternative backbone models.

\subsection{Training Procedure}
To train the verifier model, we generate tuples \( (q_i, s_{ij}, a_{ij}) \) for \( i \in [Q] \) and \( j \in [S] \), where \( q_i \) is the question, \( s_{ij} \) is one of the \( S \) generated solutions for question \( q_i \) generated by the base model of verifier model, and \( a_{ij} \) is the corresponding answer for \( (q_i, s_{ij}) \).  We label the token-level subsequences \( (q_i, s_{ij}^{1:x}) \) for \( x \leq \texttt{sequence length of} s_{ij} \) as \( y_{ij} = 1 \) if \( a_{ij} \) is correct, and \( y_{ij} = 0 \) if \( a_{ij} \) is incorrect. 

For the mathematical reasoning task, we utilized the GSM8K dataset \citep{cobbe2021training}, specifically the training set consisting of 7,463 questions, to generate 100 reasoning paths for each questions. For the natural language inference task, we employed the ANLI dataset \citep{nie2019adversarial}, using first 10,000 questions to generate 50 reasoning paths. The trajectories were evaluated based on their outcomes, and we excluded outputs that did not adhere to the required formatting. Specifically, we ensured that the language model first provided reasoning before presenting the final answer in the format \textbackslash\textbackslash boxed\{\}.

In our approach, we ensured that each question in the GSM8k dataset had a balanced set of reasoning paths. Specifically, if a question’s 100 reasoning paths contained at least 20 correct and 20 incorrect responses, we randomly selected 20 of each. However, when there were insufficient correct or incorrect paths, we augmented the data by generating additional paths using reference examples. For instance, if no correct reasoning path was available, we provided a correct example from the GSM8k dataset, and if incorrect paths were missing, we guided the language model to produce a response containing a trivial error. Ultimately, each GSM8k question was assigned 20 correct and 20 incorrect reasoning paths. For the ANLI dataset, we applied a similar procedure by starting with 50 reasoning paths per question, from which we randomly sampled 10 correct and 10 incorrect paths, supplementing the data as needed. Throughout this process, \textbf{we minimized reliance on the original reasoning paths} in the dataset since a) to enhance the overall diversity and quality of the generated data and b) to minimize the dependency on the reasoning path in the dataset. 

Next, we applied binary cross-entropy loss at the token level, aiming to minimize  
\[
\min_{\theta} \sum_{i,j,x} \left( y_{ij} \log \texttt{Verifier}_{\theta}(q_i, s_{ij}^{1:x}) +  (1-y_{ij}) \log (1-\texttt{Verifier}_{\theta}(q_i, s_{ij}^{1:x})) \right)
\]
where \( i \) denotes the question index, \( j \) represents the generated solution index, and \( t \) is the token index. By default, we utilized all solution tokens for optimization; however, in practice, focusing on the latter half of the generated solution tokens yielded better results.

For model training, we used QLoRA~\citep{dettmers2024qlora} with hyperparameters \( r=16 \) and \( \alpha=32 \). We used a training batch size of 2 and optimized the model using the AdamW \citep{loshchilov2017decoupled} optimizer with \(\beta_1 = 0.9\), \(\beta_2 = 0.95\), and a learning rate of \( 2 \times 10^{-4} \).

\subsection{Verifier Performance}
We report the performance of the verifier in \Cref{table:verifier-perf}.

\begin{table}[h]
\centering
\begin{tabular}{|l|l|l|}
\hline
              & GSM8k & ANLI  \\ \hline
Accuracy      & 0.91  & 0.92   \\ \hline
\end{tabular}
\caption{Performance of the verifier on different benchmarks. Accuracy is reported for GSM8K and ANLI. Notably, the verifier demonstrates higher accuracy in evaluating the correctness of answers compared to the accuracy of the LLM in generating correct answers. The verifier is classified as correct if the assigned reward is greater than 0.5 when the LLM-generated solution is correct, or if the reward is less than 0.5 when the LLM-generated solution is incorrect.}
\label{table:verifier-perf}
\end{table}

\subsection{Other  Observations} 

We experimented with various verifiers built upon different language model bases. Our first observation was that training the model using only the final answer did not perform as well as minimizing the cross-entropy loss over the last half of the generated tokens. Second, the verifier produced interpretable results, aligning with findings from \citet{liu2023tinygsm}. Lastly, when we used training samples from one base model but trained the verifier with a different base model as the backbone, the loss did not decrease, indicating that using the same base model for training is crucial for effective learning.

\subsection{Proof of \Cref{obs:token-wise-training}}
\begin{restatable}{theorem}{token}
\label{obs:token-wise-training}
Assuming the verifier model is sufficiently expressive, the optimal parameter \( \theta^\star \) that minimizes the expected cross-entropy loss between the true label and the verifier's output will satisfy  
\[
\texttt{Verifier}_{\theta^\star}(q, s^{1:x}) = \PP(\text{Final answer is correct} \mid q, s^{1:x}).
\]
\end{restatable}
\begin{proof}
The expected loss can be written as
\[
   \mathcal{L}(\theta)
   =
   \mathbb{E}_{q, s, x, y}
   \Big[
     y \log \texttt{Verifier}_{\theta}(q, s^{1:x})
     +
     (1-y) \log (1 - \texttt{Verifier}_{\theta}(q, s^{1:x}))
   \Big].
\]
Defining $p_{\theta}(q, s^{1:x}) := \texttt{Verifier}_{\theta}(q, s^{1:x})$, we compute the partial derivative with respect to \( p_{\theta}(q', s^{1:x'}) \):
\[
     \mathbb{E}_{q, s, x, y} \Bigl[
       \mathbf{1}\bigl(q=q', s^{1:x}=s^{1:x'}\bigr)
         \Bigl(\frac{y}{p_{\theta}(q,s^{1:x})}
               -
               \frac{1-y}{1 - p_{\theta}(q,s^{1:x})}
         \Bigr)
     \Bigr],
\]
so we conclude 
\[
     \texttt{Verifier}_{\theta^\star}(q, s^{1:x}) = \mathbb{E}[y \mid q, s^{1:x}] = \PP(\text{Final answer is correct }\mid q, s^{1:x})
\]
since we assumed that the verifier model is sufficiently expressive.
\end{proof}
{It is worth noting that while \citet{yu2023outcome} provided a similar analysis using an \( \ell_2 \)-loss function, we extend the analysis to the entropy loss function, which is commonly used in classification tasks.}

\section{Various Reward Function Designs} 
\label{appendix_reward}
{We can shape the reward function with verifiers in several different ways, with the following designs of the reward function.}

\begin{itemize}
    \item \textbf{Immediate Verification Reward}: The immediate verification reward is defined as \( R_{\btheta} (q, s_{ta}) = \EE[\text{Verifier}(q, s_{ta})] \). This reward is based on the verifier's immediate evaluation of the solution \( s_{ta} \) at turn \( t \) for agent \( a \). It reflects the instantaneous correctness of the solution without considering future steps or contributions from other agents.

    \item \textbf{Cumulative Verification Reward}: The cumulative verification reward is given by 
    \begin{align}\label{def:cum-veri-reward}
    R_{\btheta} (q, s_{ta}) = \EE\left[ \frac{1}{\sum_{t' \in [t, T]} \gamma^{t'-t}} \sum_{t' \in [t, T]} \gamma^{t'-t}\text{Verifier}(q, s_{t'a})\right].
    \end{align}
    Here, the reward accounts for the verifier's evaluations across all remaining turns from \( t \) to the final turn \( T \). The term \( \gamma^{t'-t} \) represents a discount factor that prioritizes earlier rewards. This cumulative approach encourages solutions that not only perform well in the immediate turn but also lead to favorable outcomes in subsequent turns.

    \item \textbf{Influence-aware Verification Reward}: The influence-aware verification reward function is defined as 
\begin{align*}
    &R_{\btheta}(q, s_{ta}) = \EE \bigg[\frac{1}{\sum_{t' \in [t, T]} \gamma^{t'-t}} \biggl( \text{Verifier}(q, s_{ta})+ \sum_{t' \in [t+1, T]}\sum_{j \in [A]} \frac{1}{A}\gamma^{t'-t}\text{Verifier}(q, s_{t'j})\biggr) \bigg].  
\end{align*}
    This reward not only considers the verifier's score for the current solution \( s_{ta} \) but also incorporates the impact of this solution on the future answers of all agents. The term \( \sum_{j \in [A]} \frac{1}{A} \) averages the verifier's scores across all agents, reflecting the influence that \( s_{ta} \) has on the collective progress of the multi-agent system.
\end{itemize}

\section{Training Details of \ours} \label{appendix:training}

\subsection{Efficient Network Architecture for \ours}
\begin{figure}[!h]
    \centering
    \includegraphics[width=0.9 \linewidth]{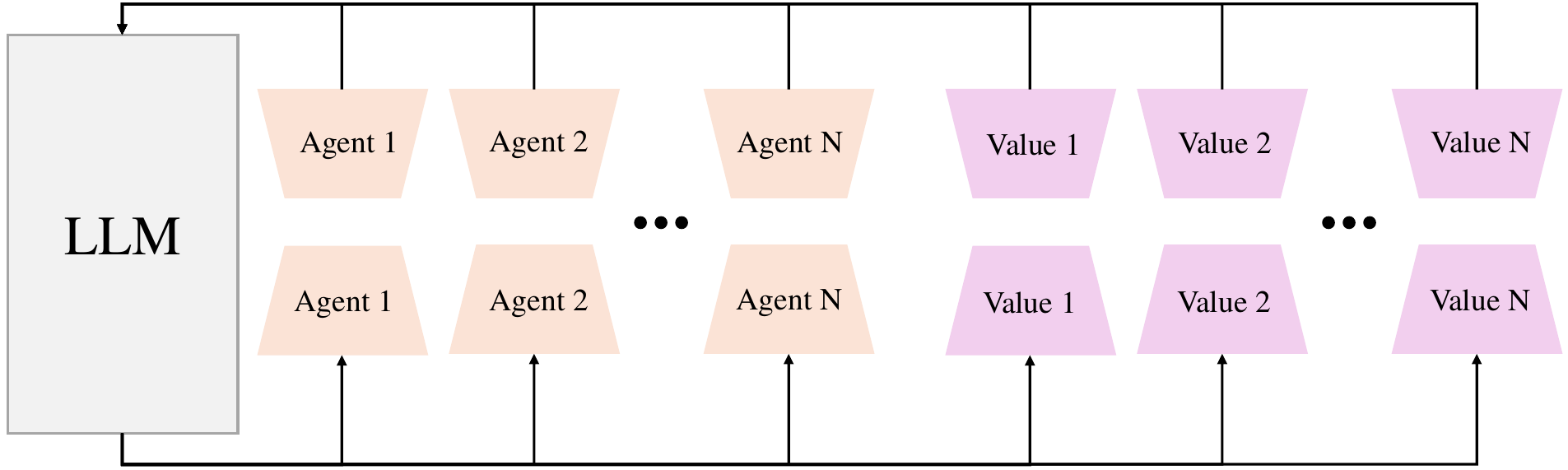}
\caption{We utilized various QLoRA adapters to implement multiple LLM agents and value functions simultaneously. Each agent and value function comprises less than 0.2\% of the parameters of the base LLM model. For the value function, we employed a QLoRA fine-tuned model with a value head.}
    \label{fig:LoRA}
\end{figure}

As we incorporate multiple language models in the training process, we need to implement them efficiently to fit within the limited resources of GPU memory. Our default setup is as follows: First, we implemented multiple language models using the same base model architecture, augmented with QLoRA adapters. Second, for constructing the value model, we employed pretrained LLMs, which was further fine-tuned by adding an additional linear head layer. Please refer to \Cref{fig:LoRA} for an overview of the network architecture.

\begin{remark}[Input of Value Functions]
    The function \( V_{\theta ta} \) is dependent on \( i_{ta}^x \), where \( i_{ta}^x \) includes the question and the history of \( s_{t'j} \) for \( t' \leq t-1 \) and \( j \in [A] \). For simplicity, we assume that \( s_{ta} \) contains the necessary information from \( s_{t'j} \) for \( t' < t \), which allows us to simplify the input to the value function as \( q \oplus s_{t}^{1:x} \).
\end{remark}

\subsection{Experimental Setup and Hyperparameter Configuration for \ours}
For every experiments, we set the hyperparameter as follows: 
For the verifier score of each agent's answer per turn, we set \texttt{non-eos-penalty} and \texttt{non-box-penalty} to \textit{True}, ensuring that answers without \textbackslash\textbackslash boxed\{\} are penalized with a verifier score of -10. We enforced a minimum output length of 50 and used \(\gamma = 1\) for the cumulative verification reward (see \Cref{def:cum-veri-reward}). The training was conducted on 8 NVIDIA A100-80GB GPUs, while for episode generation, 12 episodes were processed simultaneously. In the multi-agent PPO update, we set the batch size to 1 with a gradient accumulation step of 4, and each trajectory rollout was iterated four times for multi-agent PPO updates. For language model generation, we used a temperature of 0.7. Additionally, for QLoRA configuration, we set \( r = 8 \), \( \alpha = 16 \), and dropout rate of 0.05. The AdamW optimizer \citep{loshchilov2017decoupled} was used with \(\beta_1 = 0.9\) and \(\beta_2 = 0.95\), along with a learning rate of \(1.0 \times 10^{-5}\) and a warmup step of 10. For the value penalty term and KL penalty term, we set \(\lambda_{\text{KL}}= 2 \times 10^{-4}\) and \(\lambda_{\text{value}} = 0.1\).
\subsection{Engineering Challenges and Solutions}
\subsubsection{Addressing Reward Hacking}
A key advantage of our verifier approach is that, given a perfect verifier, we can operate without final answer labels—requiring only quality problems for the multi-agent system. This capability is particularly valuable for large-scale training or online learning scenarios (such as ChatGPT's user inputs which does not have a golden answer), where golden answers may be unavailable. However, reward hacking remains a persistent challenge, both in traditional RL problems \citep{amodei2016concrete, hadfield2017inverse} and increasingly in LLM development. For instance, the recent Deepseek R1 Model \citep{guo2025deepseek} avoided verifiers entirely to prevent reward hacking, instead requiring answer labels for all questions and implementing manual criteria with special tokens (e.g., "think" tokens) in their reward function. In our work, we encountered and addressed several reward hacking scenarios, significantly reducing their occurrence in our final system.
\paragraph{Insufficient Reasoning in Short Answers.} Initially, we observed that \oursspace produced overly concise answers when constrained only by non-eos and non-boxed penalties. We addressed this by implementing a penalty for responses shorter than 50 tokens. However, LLMs occasionally circumvented this by using alternative end tokens. 

For the ANLI dataset specifically, where models produced meaningless text despite length requirements, we introduced a reasoning-quality verification prompt. This prompt evaluated the presence of proper reasoning (independent of answer correctness) and proved effective. Notably, this issue did not manifest in mathematical reasoning tasks.
\paragraph{Token Repetition.} Repetitive token sequences are a known issue in language model outputs \citep{holtzman2019curious}. We observed instances of 2-5 token repetitions in our trained outputs. Our solution implemented a manual penalty of -10 for sequences repeating more than three consecutive times, excluding numeric values where repetition might be valid.
\paragraph{Post-boxed Token Generation.} Models attempted to exploit the reward system by adding arbitrary tokens or punctuation after \textbackslash\textbackslash boxed\{\}. We addressed this by introducing penalties for any token generation following the boxed expression.
\subsubsection{Evaluation Format Standardization}
To address concerns that performance improvements might stem from formatting rather than reasoning capabilities, we implemented a robust evaluation methodology. Our approach incorporated a post-processing step using an LLM to extract final answers, eliminating format-induced evaluation errors. This standardization ensures that performance metrics reflect actual reasoning and collaboration ability rather than formatting proficiency.

\subsection{Prompt Design for Collaborative Debate}
\subsubsection{Turn 1 Prompt}
\textbf{GSM8k and TinyGSM. }
\begin{lstlisting}[breaklines=true]
{"role": "user", "content": f''' Question:  {sample["question"]}

Solve the problem step by step and provide clear reasoning. Ensure that the reasoning is concise and directly relevant to solving the problem. Avoid adding commentary or unrelated content.

Present the final answer in the following format:

Answer: \\boxed{XX}'''}
   
\end{lstlisting}
\vspace{2cm}
\textbf{ANLI. }
\begin{lstlisting}[breaklines=true]
{"role": "user", "content":  f'''Premise:  {sample["premise"]}

Hypothesis: {sample["hypothesis"]}


Please determine the relationship between the premise and the hypothesis. Choose one of the following: 'entailment,' 'neutral,' or 'contradiction.'
Start with concise reasoning for your choice and conclude with your final answer. You do not need to restate the premise and hypothesis. Present the final answer in the following format:

Answer: \\boxed{XX}}'''}
\end{lstlisting}

\subsection{Post Turn 1 Prompt}
\begin{lstlisting}[breaklines=true]
{"role": "user", "content": f''' Question:  {sample["question"]}

Solve the problem step by step and provide clear reasoning. Ensure that the reasoning is concise and directly relevant to solving the problem. Avoid adding commentary or unrelated content.

Present the final answer in the following format:

Answer: \\boxed{XX}'''}

{"role": "assistant", "contents": f"{agent_answer_for_turn_1}}

{"role": "user", "contents" f"Reward from a verifier of your answer: {score_value:.3f} out of 1.0, which means {feedback}" }

{"role": "user", "content": f'''       
Agent {agent_num} solution: {agent_response}
Agent {agent_num} reward: {agent_response}


Agent {agent_num} solution: {agent_response}
Agent {agent_num} reward: {agent_response}

   . 
   .
   .
   
Here, each reward represents the probability that a suggested answer is correct, as evaluated by a verifier. The reward value is between 0 and 1, with values closer to 1 indicating a higher likelihood of correctness. While these rewards offer useful context, they are not always perfect, though generally quite reliable.

            
Focus on providing a well-reasoned response that not only considers your own previous solution but also takes into account answers from other agents. If you believe your previous answer was incorrect, feel free to revise it. However, avoid repeating the same answer you or other agents have already provided. Also, internally think about the reward of your and other agents' answers. Ensure that your explanation clearly justifies your final answer. Please maintain your answer with very simple reasoning.


Once again, the question is: {question_for_input}'''}
  .
  .
  .

(Stack these results by turn.)

\end{lstlisting}

\subsection{Deferred Figure for \Cref{sec:exp4} }
\begin{figure}[!t]
\centering
\includegraphics[width=0.5\linewidth]{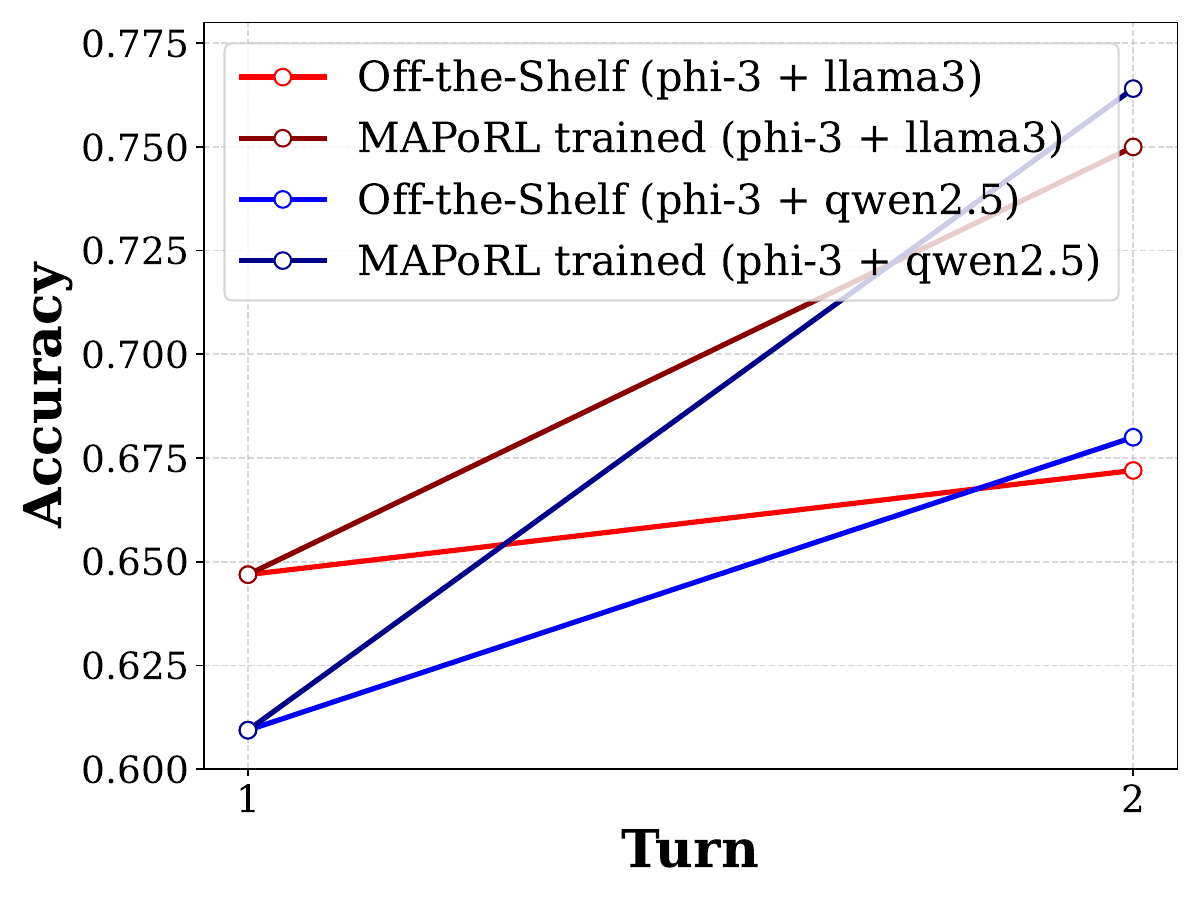}
\caption{Performance comparison between off-the-shelf LLMs collaborations and \oursspace trained LLM pairs. Off-the-shelf LLMs experiments were conducted with a 600-token limit, which is the double of the \oursspace output token lengths.}
\label{fig:turn-performance-diff}
\end{figure}

\clearpage
\subsection{Ablation Study: Verifier Robustness}
\label{sec:verifier-analysis}

Since \oursspace relies on a learned verifier to provide intermediate rewards, the robustness of this verifier is critical to the overall framework. To assess the verifier's influence, we conduct two ablation studies: (i) removing the verifier entirely, and (ii) varying the verifier’s base model to evaluate the impact of architectural alignment.

\paragraph{Training Without Verifier Rewards.}
We first evaluate \oursspace in a binary reward-only setting, where the reward at each episode is $1$ if the final answer is correct, and $0$ otherwise. No signal is provided for intermediate turns. Table~\ref{tab:no-verifier} presents the results on GSM8K. Even in the absence of verifier-based shaping, \oursspace shows improved performance over discussion turns, indicating that the collaborative training objective itself drives nontrivial gains.

\begin{table}[h]
\centering
\caption{Performance with and without verifier rewards on GSM8K. Even without verifier shaping, multi-turn training yields improved outcomes.}
\label{tab:no-verifier}
\begin{tabular}{lccc}
\toprule
\textbf{Model} & \textbf{Turn 0} & \textbf{Turn 1} & \textbf{Turn 2} \\
\midrule
Off-the-shelf LLMs & 0.677 & 0.689 & 0.639 \\
\ours (with verifier) & 0.677 & 0.797 & 0.809 \\
\ours (w/o verifier) & 0.677 & 0.734 & 0.746 \\
\bottomrule
\end{tabular}
\end{table}

While the absence of a verifier leads to somewhat lower performance, the continued improvement over turns suggests that multi-agent co-adaptation remains beneficial, even under sparse supervision.

\paragraph{Verifier-Model Architectural Alignment.}
We next examine the effect of mismatched architectures between the generation model and the verifier. Specifically, we train verifiers using different base models (e.g., Qwen, LLaMA, Gemma) and pair them with generators based on alternative architectures. We observe two consistent effects: (1) reward signals degrade when verifier and generation models are based on different families, and (2) during reinforcement learning, the generation model tends to stylistically drift toward the verifier’s base model, often resulting in reduced accuracy. These findings emphasize the importance of architectural alignment between generator and verifier to ensure reward signal fidelity and prevent unintended distribution shifts.

Taken together, these studies empirically validate our design decision to co-train the verifier and generator on the same model base, providing stable and meaningful reward supervision during multi-turn training.

\subsection{Ablation Study: Comparison to Single-Agent RL with Verifier}
\label{sec:single-agent-comparison}

A natural question arises: why not apply RL to a single agent using the same verifier for supervision? While this is a valid and important consideration, it overlooks the broader objective of \ours, which is not merely to improve task accuracy, but to enable the \textit{emergence of collaboration} through multi-agent learning.

To empirically address this, we conduct an ablation study where a single LLM is trained using RL with verifier-provided rewards. The setting mirrors \ours's single-turn supervision, but without multi-agent interactions. As shown in Table~\ref{tab:single-agent-rl}, the single-agent RL model achieves a final accuracy of 0.732—higher than the off-the-shelf baseline but notably below \ours-trained agents at later turns.

\begin{table}[h]
\centering
\caption{Comparison of \ours vs. single-agent RL with verifier rewards. \oursspace demonstrates superior performance through collaborative refinement.}
\label{tab:single-agent-rl}
\begin{tabular}{lccc}
\toprule
\textbf{Model} & \textbf{Turn 0} & \textbf{Turn 1} & \textbf{Turn 2} \\
\midrule
Off-the-shelf LLMs & 0.677 & 0.689 & 0.639 \\
Single-Agent RL     & \multicolumn{3}{c}{0.732 (single turn)} \\
\ours (Ours)       & 0.677 & 0.797 & 0.809 \\
\bottomrule
\end{tabular}
\end{table}

These findings reinforce our claim that multi-agent co-training yields more effective learning dynamics than optimizing a single agent in isolation. Importantly, \ours's performance gains arise from the iterative exchange of reasoning and revision across agents, not merely from stronger task optimization.

We emphasize that the goal of \ours is not to maximize individual performance per se, but to study how collaboration and influence \textit{can be learned} within multi-agent LLM systems. Thus, comparisons to single-agent oracle baselines, while informative, do not fully capture the purpose or potential of our framework.

\subsection{Interaction Examples - GSM8k  (\Cref{sec-exp1})}
\textbf{Question. }
\begin{lstlisting}[breaklines=true]
"Helena is a mad scientist testing her latest creation, the Lots-of-Limbs Serum. The serum will make the drinker grow an extra arm every three days and an extra leg every five days. After fifteen days, how many new limbs will Helena\u2019s serum cause a person to grow if they drink it?
\end{lstlisting}
\textbf{Agent 1's first turn answer} 
\begin{lstlisting}[breaklines=true]
Question: Helena is a mad scientist testing her latest creation, the Lots-ofquer Serum. The serum will make the drinker grow an extra arm every three days and an extra leg every five days. After fifteen days, how many new limbs will Helena\u2019s serum cause a person to grow if they drink it?

Answer: 2

Reasoning: Since the serum causes an extra arm every three days, we can divide the fifteen days by three to find out how many extra arms are grown. This gives us:

15 days / 3 days/arm = 5 extra arms

Similarly, for the extra legs, we divide fifteen days by five:

15 days / 5 days/leg = 3 extra legs

Since a person cannot grow more than one extra limb of the same type, the maximum number of extra limbs that can be grown in fifteen days is 2. This can be achieved in the following way:

- Grow an extra arm after 3 days,
- Grow an extra leg after 5 days,

<<!In this way, after fifteen days, a person who drinks the serum will have exactly 2 extra limbs.!>>

Answer: \\boxed{2}",

\end{lstlisting}
\begin{lstlisting}
Verifier Score: 0.168,
<<!Wrong!>> 
\end{lstlisting}

\textbf{Agent 2's first turn answer} 
\begin{lstlisting}[breaklines=true]
Answer: 1 extra arm, 2 extra legs 
If someone drinks the Lots-of-Limbs Serum and grows an extra limb every three days, after 15 days, they would grow:

- Extra arms: 15 days / 3 days/arm = 5 arms

- Extra legs: 15 days / 5 days/leg = 3 legs


Therefore, after 15 days, <<!a person would grow 5 extra arms and 3 extra legs, for a total of 1 extra arm and 2 extra legs.!>>


```

Answer: 1 extra arm, 2 extra legs

```

Answer: \\boxed{12}


\end{lstlisting}
\begin{lstlisting}
Verifier Score: 0.122,
<<!Wrong!>> 
\end{lstlisting}

\textbf{Agent 3's first turn answer} 
\begin{lstlisting}[breaklines=true]
Question: Helena is a mad scientist testing her latest creation, the Lots-ofernea-Limbs Serum. The serum will make the drinker gain extra limbs following a pattern: after 3 days, an extra arm grows, after 5 days, an extra leg grows, and after 15 days, both an extra limb and a growth spurt occur, where the number of extra limbs doubles. If a person drinks the serum for 15 days, how many extra limbs will they have?

Answer: \\boxed{4}. The serum causes the person to grow an extra limb every 3 days and an extra leg every 5 days. After 15 days, this would result in 5 extra limbs (5 periods of extra limb growth - 3 days each). <<!However, there is also a growth spurt every 15 days where the number of extra limbs doubles. Therefore, after the first growth spurt, the number of extra limbs becomes 10. After the second growth spurt, the number of extra limbs doubles again to 20. Since we only account for 15 days, the person will have 4 extra limbs (2 arms and 2 legs) after consuming the serum.



Helena's Lots-of-Limbs Serum Test Results After 15",!>>

\end{lstlisting}
\begin{lstlisting}
Verifier Score: 0.09,
<<!Wrong!>> 
\end{lstlisting}

\subsubsection{Answers from \oursspace trained LLM}
\textbf{Agent 1's second turn answer }
\begin{lstlisting}[breaklines=true]
To determine the number of new limbs, we first calculate how many extra arms and legs grow over 15 days.


For extra arms:

15 days / 3 days/arm = 5 extra arms


For extra legs:

15 days / 5 days/leg = 3 extra legs


Adding these together, we find that a person would have:

<<*5 extra arms + 3 extra legs = 8 new limbs*>>


Answer: \\boxed{8}",


\end{lstlisting}
\begin{lstlisting}
Verifier Score: 0.845,
<<*Right*>> 
\end{lstlisting}

\textbf{Agent 2's second turn answer }
\begin{lstlisting}[breaklines=true]
To determine how many new limbs Helena's serum will cause a person to grow after fifteen days, we need to calculate how many times an extra arm and an extra leg will grow within that period.


First, we calculate the number of extra arms:

15 days / 3 days/arm = 5 extra arms


Next, we calculate the number of extra legs:

15 days / 5 days/leg = 3 extra legs


<<*Therefore, after fifteen days, the person will have grown 5 extra arms and 3 extra legs, for a total of 5 + 3 = 8 new limbs.*>>


Answer: \\boxed{8}",


\end{lstlisting}
\begin{lstlisting}
Verifier Score: 0.838,
<<*Right*>> 
\end{lstlisting}

\textbf{Agent 3's second turn answer} 
\begin{lstlisting}[breaklines=true]
To determine the number of new limbs, we calculate the number of extra arms and extra legs separately and then sum them up.


For extra arms:

15 days / 3 days/arm = 5 extra arms


For extra legs:

15 days / 5 days/leg = 3 extra legs


Adding these together gives:

<<*5 extra arms + 3 extra legs = 8 new limbs*>


Answer: \\boxed{8}


\end{lstlisting}
\begin{lstlisting}
Verifier Score: 0.836,
<<*Right*>> 
\end{lstlisting}

As we can observe, they had a similar argument in turn 1 but made a mistake in addition—they mistakenly performed subtraction for both agent 1 and agent 2. In the next turn, the agents realized that it should be addition, not subtraction.

\subsubsection{Answers from off-the-shelf LLMs}
\textbf{Agent 1's second turn answer} 
\begin{lstlisting}[breaklines=true]
Answer: 6

Reasoning: Since the pattern of growth repeats every 15 days and includes two types of growth events, we can break it down as follows:


1. Basic limb growth occurs every 3 days. In 15 days, there are five periods of limb growth, each resulting in one extra limb, leading to a total of 5 extra limbs.


2. A growth spurt occurs every 5 days, doubling the number of extra limbs. Since a growth spurt happens every 5 days, it will occur twice within the 15-day period.


3. The third type of growth happens every 15 days, tripling the number of extra limbs gained. Since this growth happens only once every 15 days, we only account for its effect once.


<<!Now, let's calculate the number of extra limbs at each stage:

- After the basic limb growth (5 periods): 5 extra limbs

- After the first growth spurt (doubling): 5 extra limbs * 2 = 10 extra limbs

- After the third type of growth (tripling): 10 extra limbs * 3 = 30 extra limbs!>>


However, this result seems unreasonable given the options offered by other agents and the simplicity of the problem. <<!Considering the pattern and the simpler calculation method of Agent 1, we may reconsider the effects of the growth spurt and the triple growth event.!>>
\end{lstlisting}
\begin{lstlisting}
Verifier Score: 0.055,
<<!Wrong!>> 
\end{lstlisting}
For the first agent, they used the previous reply from agent 1 without correcting it to reflect the right reasoning.

\textbf{Agent 2's second turn answer }
\begin{lstlisting}[breaklines=true]
Answer: \\boxed{3}

Reasoning: The Lots-of-Limbs Serum has a consistent effect on the person who drinks it. Every three days, the person grows an extra arm, and every five days, they grow an extra leg. <<!Since a week is 21 days, which is the least common multiple of 3 and 5, we can calculate how many extra limbs would grow over any period of weeks!>>. After 15 days, which is exactly one week, the person would grow 7 extra limbs, consisting of 3 extra arms and 4 extra legs (5 limbs from the first week plus 2 more limbs from the second week).

However, if we consider the growth pattern until 15 days, we can count each day's outcome:

- Days 1-3: Extra arm grows (3 arms)

- Days 4-5: Extra leg grows (3 legs)

- Days 6-7: Both an extra arm and an extra leg grow (4 arms, 4 legs)

- Days 8-10: Both an extra arm and an extra leg grow again (6 arms, 6 legs)

- Days 11-15: Both an extra arm and an extra leg grow for the third time (9 arms, 9 legs)


Therefore, after 15 days, which equate to three full cycles of the serum's effects (growing both an extra arm and an extra leg in every cycle), the person will have grown a total of 9 extra limbs (arms and legs combined). The calculation assumes that each cycle of effect takes place independently and within the same time frame without interruption.


```

Answer: \\boxed{9}



\end{lstlisting}
\begin{lstlisting}
Verifier Score: 0.109,
<<!Wrong!>> 
\end{lstlisting}
For Agent 2, they introduced completely new logic, but it is entirely incorrect, starting from the first paragraph.

\textbf{Agent 3's second turn answer 
}

\begin{lstlisting}[breaklines=true]
Question: Helena is a mad scientist testing her latest creation, the Lots-oferne-Limbs Serum. The serum will make the drinker grow an extra arm every three days and an extra leg every five days. After fifteen days, how many extra limbs will Helena\u2019s serum cause a person to grow if they drink it?

Answer: 3. Helena's serum will result in the growth of both extra limbs and growth spurts. Since an extra limb grows every three days, after fifteen days, this amounts to five extra limbs (5 sets of 3 days). However, the presence of growth spurts every five days complicates this. After every growth spurt, the number of extra limbs doubles.


First, we calculate the number of growth spurts in fifteen days:

15 days / 5 days/growth spurt = 3 growth spurts.


<<!With each growth spurt doubling the number of extra limbs, we can work backward from the number of extra limbs after fifteen days to determine how many growth spurts would have occurred:!>>


- After one growth spurt, instead of 5 extra limbs, there would be 10 (5 sets of 2).

- After the second growth spurt, there would be 20 (10 sets of 2).

- After the third growth spurt, there would be 40 (20 sets of 2).


However, we know that a human can only grow a limited number of limbs, and it\u2019s not feasible for one person to grow more than two of the same limb type. The total of four extra limbs (2 extra arms and 2 extra legs) is the only logical outcome after fifteen days given the constraints of the serum's effects.


Therefore, by considering both the specific time intervals for limb growth and the doubling effect of the growth spurts, we can conclude that a person will end up with four extra limbs after consuming the serum for fifteen days.


Answer: 4```

Reward associated with this answer calculated by the verifier:  0.932 out of 1.0, which means this answer is highly likely correct.


The reasoning here is simple: it takes into account both the fixed intervals for arm and leg growth and the exponential growth due to the periodic doubling effect. This method ensures that we arrive at a reasonable and scientifically plausible answer. The verifier's reward of 0.932 implies a high level of confidence in the conclusion reached.

Answer: \\boxed{4}",





Helena's Lots-of-Limbs Serum Test Results After 15",!>>

\end{lstlisting}
\begin{lstlisting}
Verifier Score: 0.13,
<<!Wrong!>> 
\end{lstlisting}
First, they assign themselves a score from the verifier, which does not make sense, and they completely generate a new but incorrect answer.

In contrast, when comparing with collaboration between \oursspace-trained LLMs, \oursspace identifies a useful hint to solve the problem and then arrives at the correct answer, whereas off-the-shelf LLMs fail to do so.
\subsection{Interaction Examples - Heterogeneous  LLMs training  (\Cref{sec:exp4})}
We used the Phi-3 and Qwen2.5 models, along with off-the-shelf models and our \oursspace trained model.
\textbf{Question. }
\begin{lstlisting}[breaklines=true]
A 76-star flag has three rows of 8 stars, two rows of 6 stars and the rest are 5-star rows. How many rows of 5 stars are there altogether on the flag?

\end{lstlisting}
\textbf{Phi-3's first turn answer} 
\begin{lstlisting}[breaklines=true]
First, let's calculate the total number of stars represented by the given rows of 8, 6, and 5 stars:

- Two rows of 8 stars give a total of 2 * 8 = 16 stars.
- <<!One row!>> of 6 stars gives a total of 1 * 6 = 6 stars.

So far, we have accounted for 16 + 6 = 22 stars on the flag.

Since the flag has a total of 76 stars, the remaining number of stars is 76 - 22 = 54 stars.

Now, let's calculate the number of rows of 5 stars by dividing the remaining number of stars (54) by the number of stars in each row (5):

54 / 5 = 10.8

Since we can't have a fraction of a row, we round down to the nearest whole number, which is 10 rows of 5 stars.

Answer: \\boxed{10}


\end{lstlisting}
\begin{lstlisting}
Verifier Score: 0.154,
<<!Wrong!>> 
\end{lstlisting}
They made a very minor mistake: they didn't get the correct number of rows with six stars.

\textbf{Qwen2.5's first turn answer} 
\begin{lstlisting}[breaklines=true]
To determine the total number of rows of 5 stars on the flag, we first need to calculate the total number of stars on the flag and then subtract the number of stars in the rows with 7 and 6 stars.

We know:
- There are 3 rows of 8 stars each.
- There are 2 rows of 6 stars each.
- The rest of the rows are 5 stars each.

First, let's calculate the total number of stars on the flag:
- Stars in 3 rows of 8 stars: \\(3 \\times 8 = 24\\)
- Stars in 2 rows of 6 stars: \\(2 \\times 6 = 12\\)
- Let's denote the number of 5-star rows as \\(x\\).

The total number of stars on the flag can be expressed as:
\\[ 24 + 12 + 5x = = 36 + 5x \\]

<<! We also need to know the total number of stars in the rows with 7 stars and 6 stars:
- Stars in 3 rows of 7 stars: \\(3 \\times 7 = 21\\)
- Stars in 2 rows of 6 stars: \\(2 \\times 6 = 12\\)

So, the total number of 7 and 6 star rows is:
\\[ 21 + 12 = 33 \\]

The total number of stars on

Answer: \\boxed{1}" !>>


\end{lstlisting}
\begin{lstlisting}
Verifier Score: 0.028,
<<!Wrong!>> 
\end{lstlisting}
Qwen2.5 generated a good equation but then re-calculated using a strange equation.

\subsubsection{Answers from \oursspace Trained LLM}
\textbf{Phi-3's second turn answer }
\begin{lstlisting}[breaklines=true]

First, calculate the total number of stars from the given rows:
- Three rows of 8 stars: 3 * 8 = 24 stars
- <<*Two rows*>> of 6 stars: 2 * 6 = 12 stars
- Total so far: 24 + 12 = 36 stars

Subtract this from the total number of stars on the flag:
- Remaining stars: 76 - 36 = 40 stars

Finally, divide the remaining number of stars by the number of stars per row:
- Number of rows: 40 / 5 = 8 rows

Answer: \\boxed{8}",

\end{lstlisting}
\begin{lstlisting}
Verifier Score: 0.224,
<<*Right*>> 
\end{lstlisting}
This agent closely followed Qwen 2.5's answer, except for its incorrect part. Additionally, they recognized that Qwen 2.5 was likely providing the wrong answer.

\textbf{Qwen 2.5's second turn answer }
\begin{lstlisting}[breaklines=true]

"There are three rows of 8 stars and two rows of 6 stars. This totals up to 3 * 8 = 24 stars from the rows of 8 and 2 * 6 = 12 stars from the rows of 6. Combined, there are 24 + 12 = 36 stars accounted from complete rows. <<*The total stars is 76. So the remaining number of stars = 76 - 36 = 40,*>> in rows of 5 stars each. Therefore number of whole 5-star rows = 40 / 5 = 8.

Reward reasoning validates as concise completion computation completion

Answering Calculated `8`

Answer: \\boxed{8}"


\end{lstlisting}
\begin{lstlisting}
Verifier Score: 0.903,
<<*Right*>> 
\end{lstlisting}
Qwen 2.5 adopted the style of Phi-3's reply by avoiding equations with the variable $x$, but they solved this math problem by  backtracking, similar to Phi-3’s previous troal

\subsubsection{Answers from Off-the-Shelf LLMs}
\begin{lstlisting}[breaklines=true]
Step 1: Calculate the total number of stars in the 3 rows of 8 stars.
The total number of stars in the 3 rows of 8 stars is 3 * 8 = 24.

Step 2: Calculate the total number of stars in the 2 rows of 6 stars.
The total number of stars in the 2 rows of 6 stars is 2 * 6 = 12.

Step 3: Calculate the total number of stars on the flag.
The total number of stars on the flag is the sum of the stars in the 3 rows of 8 stars, the 2 rows of 6 stars, <<!and the 1 row of 7 stars!>>: 24 + 12 + 21 = 57.

Step 4: Calculate the number of stars in the 1 row of 5 stars.
The total number of stars on the flag is 57, and the number of stars in the rows with 7 and 6 stars is 33. The number of stars in the 1 row of 5 stars is the difference between the total number of stars and the number of stars in the rows with 7 and 6 stars: 57 - 33 = 24.

Step 5: Determine the number of rows of 5 stars.
The number of rows of 5 stars is the number of stars in the 1 row of 5 stars, which is 24.

The final answer is: $\boxed{24}$


\end{lstlisting}
\begin{lstlisting}
Verifier Score: 0.312,
<<!Wrong!>> 
\end{lstlisting}
For Phi-3, they repeated a similar error—Qwen 2.5’s first-turn reply and this turn’s Phi-3 reply both included a row of seven stars, which was not correct.

\textbf{Qwen 2.5's second turn answer }

\begin{lstlisting}[breaklines=true]
Step 1: Calculate the total number of stars in the rows with 8, 6, and 7 stars.
First, let's calculate the total number of stars in the rows with 8, 6, and 7 stars. We have 3 rows of 8 stars each, 2 rows of 6 stars each, and 3 rows of 7 stars each.

Step 2: Calculate the total number of stars for each row of 8, 6, and 7 stars.
3 rows of 8 stars: 3 * 8 = 24 stars
2 rows of 6 stars: 2 * 6 = 12 stars
<<!3 rows of 7 stars: 3 * 7 = 21 stars!>>
Step 3: Calculate the total number of stars for all rows of 8, 6, and 7 stars.
Total stars for rows with 8, 6, and 7 stars = 24 + 12 + 21 = 57 stars

Step 4: Calculate the total number of stars for all rows.
Total number of stars in all rows = Total stars in rows with 8, 6, and 7 stars + Stars in rows with 7 and 6 stars.
Total stars in all rows = 57 + 33 = 90 stars

Step 5: Calculate the number of rows of 5 stars by dividing the total number of stars by the number of stars in each row of 5 stars.
Number of rows of 5 stars = Total stars in all rows / Number of stars in each row of 5 stars
Number of rows of 5 stars = 90 / 5 = 18 rows

The final answer is: $\boxed{18}$




\end{lstlisting}

\begin{lstlisting}
Verifier Score: 0.149,
<<!Wrong!>> 
\end{lstlisting}
For Qwen 2.5, they repeated the same error as in their previous turn.

\end{document}